\documentclass[lettersize,journal]{IEEEtran}

\usepackage{amsmath,amsfonts}
\usepackage{algorithmic}
\usepackage{array}
\usepackage[caption=false,font=normalsize,labelfont=sf,textfont=sf]{subfig}
\usepackage{textcomp}
\usepackage{stfloats}
\usepackage{url}
\usepackage{verbatim}
\usepackage{graphicx}
\def\BibTeX{{\rm B\kern-.05em{\sc i\kern-.025em b}\kern-.08em
    T\kern-.1667em\lower.7ex\hbox{E}\kern-.125emX}}
\usepackage{balance}

\usepackage{orcidlink}
\usepackage{hyperref}
\usepackage{lipsum}
\newcommand{\eg}{{\ignorespaces\emph{e.g.,}}{ }}
\newcommand{\ie}{{\ignorespaces\emph{i.e.,}}{ }}

\usepackage{amsmath,amsfonts,bm}

\def\eqref#1{(\ref{#1})}
\def\Eqref#1{Equation~(\ref{#1})}

\def\1{\bm{1}}

\def\rd{{\textnormal{d}}}

\def\mQ{{\bm{Q}}}

\def\mS{{\bm{S}}}

\def\mU{{\bm{U}}}

\DeclareMathAlphabet{\mathsfit}{\encodingdefault}{\sfdefault}{m}{sl}
\SetMathAlphabet{\mathsfit}{bold}{\encodingdefault}{\sfdefault}{bx}{n}

\def\calL{{\mathcal{L}}}

\newcommand{\E}{\mathbb{E}}

\newcommand{\R}{\mathbb{R}}

\def\dt{{ \mathrm{d} t }}

\newcommand{\nicefracpartial}[2]{\nicefrac{\partial #1}{\partial  #2}}

\newcommand{\fracpartial}[2]{\frac{\partial #1}{\partial  #2}}
\newcommand{\fracdiff}[2]{\frac{\rd #1}{\rd  #2}}
\newcommand{\br}[1]{\left[#1\right]}
\newcommand{\pr}[1]{\left(#1\right)}
\newcommand{\T}{\top}
\newcommand{\norm}[1]{\left\Vert#1\right\Vert}

\newcommand{\diag}{\mathrm{diag}}
\newcommand{\vectorize}{\mathrm{vec}}
\newcommand{\Inv}{\dagger}

\newcommand*\Qx{ {Q_k^x} }%
\newcommand*\Qu{ {Q_k^\theta} }%
\newcommand*\Qxx{ {Q_k^{xx}} }%
\newcommand*\Qxu{ {Q_k^{x\theta}} }%
\newcommand*\Qux{ {Q_k^{\theta x}} }%
\newcommand*\Quu{ {Q_k^{\theta\theta}} }%

\newcommand*\fx{ {f_k^x} }%
\newcommand*\fu{ {f_k^\theta} }%
\newcommand*\fxx{ {f_k^{xx}} }%
\newcommand*\fxu{ {f_k^{x\theta}} }%
\newcommand*\fuu{ {f_k^{\theta\theta}} }%

\newcommand*\lu{ {\gamma \bar \theta_k} }%
\newcommand*\luu{ {\gamma} }%

\newcommand*\Vx{ {V_{k+1}^x} }%
\newcommand*\Vxx{ {V_{k+1}^{xx}} }%

\newcommand*\dxt{ {\delta x_{k}} }%
\newcommand*\dut{ {\delta \theta_{k}} }%

\newcommand*\Fx{ {\fracpartial{F}{x_t}} } %
\newcommand*\Fu{ {\fracpartial{F}{\theta}} }%

\usepackage{nicefrac}  %
\usepackage{mathtools} %
\newcommand\numberthis{\addtocounter{equation}{1}\tag{\theequation}}

\usepackage{amsthm} %

\usepackage{thmtools}
\usepackage{thm-restate}
\declaretheorem[
  style=boldrestatable,
  name=Theorem
]{thm}

\declaretheorem[name=Lemma, numberlike=thm]{lemma}

\usepackage{algorithm}
\usepackage{algorithmic}
\usepackage{bbm}

\usepackage{color,soul}
\colorlet{color1}{green!50!black}
\colorlet{color2}{orange!95!black}
\colorlet{color3}{red!80!black}
\colorlet{color4}{red!65!black}
\colorlet{color5}{blue!75!green}
\colorlet{blueee}{blue!50!black}
\definecolor{nicerred}{HTML}{99292A} %
\definecolor{niceyellow}{HTML}{D89A3C} %
\definecolor{nicerblue}{HTML}{417481} %

\usepackage{booktabs} %
\usepackage{multirow} %
\newcommand{\specialcell}[2][c]{\begin{tabular}[#1]{@{}c@{}}#2\end{tabular}}
\newcommand{\specialcellr}[2][r]{\begin{tabular}[#1]{@{}r@{}}#2\end{tabular}}
\newcommand{\specialcelll}[2][l]{\begin{tabular}[#1]{@{}l@{}}#2\end{tabular}}

\usepackage{thmtools}

\hypersetup{
    colorlinks,
    linkcolor={nicerblue},
    citecolor={blue!50!black},
    urlcolor={blue!80!black}
}

\begin{document}

\title{Optimal Control Theoretic Neural Optimizer: \\ From Backpropagation to Dynamic Programming}

\author{%
  Guan-Horng Liu \orcidlink{0000-0002-8989-7568},
  Tianrong Chen \orcidlink{0000-0002-7966-5288}, and
  Evangelos A. Theodorou \orcidlink{0000-0002-9371-9633}
\thanks{Manuscript created \today.
This research was supported by the ARO Award  \#W911NF2010151 and DoD Basic Research Office Award HQ00342110002.
\textit{(Corresponding author: Guan-Horng Liu.)}

Guan-Horng Liu is affiliated with FAIR, Meta. Evangelos Theodorou is affiliated with the School of Aerospace Engineering, Georgia Institute of Technology. Tianrong Chen is affiliated with Apple MLR. 
Email: \\ 
{ghliu@gatech.edu, tchen54@apple.com, evangelos.theodorou@gatech.edu}

Work was done while Guan and Tianrong were at Georgia Tech.%
}}

\markboth{Journal of \LaTeX\ Class Files,~Vol.~18, No.~9, September~2020}%
{Optimal Control Theoretic Neural Optimizer: From Backpropagation to Dynamic Programming}

\maketitle

\begin{abstract}
  Optimization of deep neural networks (DNNs) has been a driving force in the advancement of modern machine learning and artificial intelligence.
  With DNNs characterized by a prolonged sequence of nonlinear propagation,
  determining their optimal parameters given an objective naturally fits within the framework of Optimal Control Programming.
  Such an interpretation of DNNs as dynamical systems has
  proven crucial in offering a theoretical foundation for principled analysis from numerical equations to physics.
  In parallel to these theoretical pursuits, this paper focuses on an algorithmic perspective.
  Our motivated observation is the striking algorithmic resemblance between the Backpropagation algorithm for computing gradients in DNNs and the optimality conditions for dynamical systems, expressed through another backward process known as dynamic programming.
  Consolidating this connection, where Backpropagation admits a variational structure, solving an approximate dynamic programming up to the first-order expansion leads to a new class of optimization methods exploring higher-order expansions of the Bellman equation.
  The resulting optimizer, termed Optimal Control Theoretic Neural Optimizer (OCNOpt),
  enables rich algorithmic opportunities, including layer-wise feedback policies, game-theoretic applications, and higher-order training of continuous-time models such as Neural ODEs.
  Extensive experiments demonstrate that OCNOpt improves upon existing methods in robustness and efficiency while maintaining manageable computational complexity,
  paving new avenues for principled algorithmic design grounded in dynamical systems and optimal control theory.
\end{abstract}

\begin{IEEEkeywords}
Optimization, deep neural network, optimal control, differential dynamic programming, backpropagation.
\end{IEEEkeywords}

\section{Introduction}
\IEEEPARstart{T}{he} last few decades have marked the exploding era of deep learning.
By constructing a nonlinear transformation interconnected through a long sequence of layer propagation, deep neural networks (DNNs) has showcased remarkable capabilities in recognizing intricate patterns from high-dimensional data \cite{lecun2015deep}, making them instrumental in addressing problems such as computer vision \cite{he2017mask}, distribution modeling \cite{goodfellow2014generative}, and even deciphering mathematical structures \cite{fawzi2022discovering}.

Numerous efforts have been dedicated to refining the architectural design of DNNs to better leverage domain-specific structures.
Examples include the integration of residual connections for pattern recognition \cite{he2016deep}, exploring their continuous-time counterparts for time-series data \cite{chen2018neural}, and adapting geometry, such as graph or symmetry \cite{bronstein2017geometric}.
Despite these rapid advancements in architectural innovation, the development of optimization methods for DNNs has progressed at a comparatively slower pace, and algorithms like Backpropagation \cite{lecun1988theoretical}, dating back in the 1980s, remain the predominant choice for training DNNs.
This contrasts dramatically with practices outside of the machine learning area, in fields such as operations research \cite{heyman2004stochastic}, aerospace \cite{trelat2012optimal}, and decision-making theory \cite{bertsekas2012dynamic}, where different kinds of systems often necessitate distinct optimization methods and methodologies to handle.

\begin{figure}[!t]
  \centering
  \includegraphics[width=.95\linewidth]{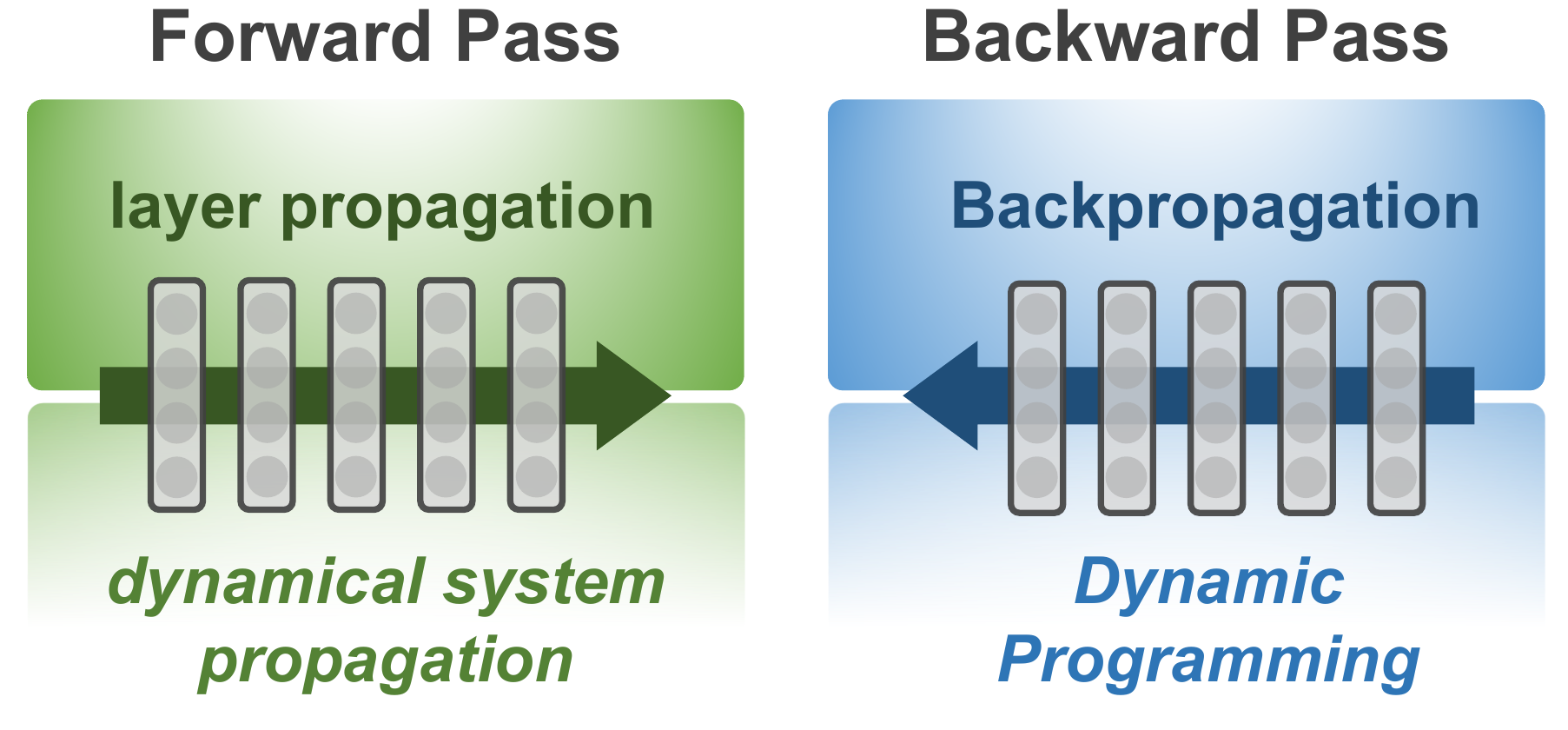}
  \caption{
    We propose a new class of training optimizers for DNNs, termed \textbf{O}ptimal \textbf{C}ontrol Theoretic \textbf{N}eural \textbf{Opt}mizer (\textbf{OCNOpt}), that interpret the forward pass as dynamical system propagation and, from which, explore the variational structure of Backpropagation as solving dynamic programming.
  }
  \label{fig:1}
\end{figure}

The aforementioned rational underscores the need for optimization methods tailored to diverse scientific structures and raises open questions that merit exploration.
Firstly, it remains intriguing to ascertain whether such a fundamental discipline, if any, unifies algorithmic connection to existing training methods and how it facilitates new applications and analyses that are otherwise challenging to foresee.
Practically, understanding how such a perspective enhances training over current methods is also of keen interest.
This paper undertakes an exploratory attempt in addressing these questions.

Specifically, our mathematical framework is anchored in the interpretation of DNNs as nonlinear dynamical systems, viewing  each layer as a distinct time step \cite{weinan2017proposal}.
This dynamical system perspective accentuates the sequential transformation inherent in modern DNNs---an essential element that remains consistent across modern deep architectures, yet is flexible enough to accommodate distinct domain-specific structures.
Futhermore, it enables systematic analysis spanning from numerical equations \cite{lu2018beyond} to principles in physics \cite{greydanus2019hamiltonian}.
For example, by casting residual networks \cite{he2016deep} as a discretization of ordinary differential equations, we gain valuable insights into the loss landscape \cite{lu2020mean} and spark the creation of novel architectures
\cite{chang2018reversible}.

This dynamical system perspective also prompts control-theoretic analysis, which further interprets the network weights as control variables.
In this light, the training process can be cast as an Optimal Control Programming (OCP), where both methodologies aim to optimize certain variables (weights \emph{vs.} controls) under the chain structure (DNNs \emph{vs.} dynamical systems).
The OCP interpretation has been explored recently to provide theoretical characterizations of the learning process \cite{han2019mean,seidman2020robust} and inspire algorithms for hyper-parameter adaptation \cite{li2017stochastic} and computational acceleration \cite{gunther2020layer,zhang2019you}.

On the optimization side, however,
OCP-inspired optimizers still face limitations, often confined to small-size dataset \cite{li2017maximum}.
Given that the necessary condition of OCP, well-known as the dynamic programming principle \cite{bellman1954theory}, shares a striking algorithmic resemblance to the Backpropagation process---both involving backward computation from the last time step or layer---a fundamental understanding between OCP and existing optimization methods for training DNNs is imperative to enhance practical applicability of OCP-inspired optimizers.

In this work, we make a significant advance toward the development of optimal control-theoretic training methods, drawing inspiration from OCP perspectives while maintaining algorithmic proximity to existing optimizers.
Our investigation unveils that Backpropagation, in essence, solves an approximate dynamic programming with associated Bellman objectives expanded only up to the first order.
This results in the omission of second-order derivatives, critical for enabling control-theoretic applications \cite{pan2015robust}.
By preserving these second-order derivatives during the training of DNNs, we introduce layer-wise feedback gains, game-theoretic (\eg bandit \cite{seldin2014one}) integration, and joint optimization of the architectures, showcasing improvements over prior methods in robustness and efficiency.
As the resulting method stands as the optimal control-theoretic generalization to prior approaches,
we term it \textbf{Optimal Control Theoretic Neural Optimizer (OCNOpt)}.

OCNOpt is the first OCP-inspired optimizer scalable to general deep architectures, including feedforward, convolutional, residual networks, and Neural Ordinary Differential Equations \cite{chen2018neural}.
Unlike traditional OCP methods reliant on expansive higher-order derivatives,
OCNOpt employs efficient second-order computations.
Specifically, we employ Gauss-Newton approximation to the last layer and show that such a low-rank factorization propagates backward through computations of all other second-order derivatives without imposing additional assumptions.
Inspired by the practical successes of existing methods, we further adapt similar curvature estimation techniques to the Bellman equation's optimization landscape.
These approximations reduce the complexity of OCNOpt by orders of magnitude compared to OCP-inspired baselines and enable significantly stabler training in practice.

In summary, we present the following contributions.
\begin{itemize}
  \item We establish a novel algorithmic perspective linking DNN training to Optimal Control Programming (OCP), unveiling for the first time a fundamental connection between existing training methods and OCP via the Differential Dynamic Programming (DDP) algorithm, without limiting to specialized settings as in prior attempts \cite{li2017maximum,jin2020pontryagin}.

  \item We introduce \textbf{OCNOpt}, a new class of training optimizers characterized by a backward computation similar to Backpropagation, but inheriting control-theoretic principles from the DDP framework.

  \item Compared to the DDP algorithm, OCNOpt achieves more efficient and stable training through computation simplification enabled by low-rank factorization and curvature approximation similar to existing training methods, with the latter arising exclusively from our framework.
  
  \item  We demonstrate the versatility of OCNOpt in training various DNN architectures, ranging from feedforward networks to Neural ODEs, leveraging game theoretic extensions and continuous-time OCP methodologies.

  \item Through extensive experiments, OCNOpt demonstrates competitive performance in tasks such as image classification, time-series prediction, and probabilistic modeling, showcasing improved convergence and robustness compared to standard optimizers.
\end{itemize}

\section{Related Work}

\noindent
The conceptualization of DNNs as dynamical systems, initiated by Weinan's proposal \cite{weinan2017proposal}, is foundational and practically versatile.
It serves as a pivotal basis for both theoretical investigations into the success of modern deep architectures \cite{lu2020understanding} and the development of innovative architectures \cite{lu2018beyond, greydanus2019hamiltonian}.
Concurrently, this perspective facilitates analyses grounded in numerical equations and physics principles, inspiring continuous-time models like Neural ODEs \cite{chen2018neural} and Neural Stochastic Differential Equations \cite{li2020scalable}, among others.
These advancements extend the capabilities of deep learning models for addressing new challenges from time-series prediction \cite{kidger2020neural} and generative modeling \cite{dhariwal2021diffusion} to climate change \cite{kashinath2021physics}.
Our work contributes to this trajectory, introducing a new research area by incorporating optimal control theory into the development of practical training algorithms.

On the other hand, the Optimal Control Programming (OCP) emerges naturally from the dynamical system perspective, offering fundamental characterizations in optimization processes.
Beyond theoretical explorations \cite{han2019mean,seidman2020robust},
OCP has demonstrated versatility in algorithmic development, inspiring methods for parallel propagation \cite{gunther2020layer,zhang2019you}, hyper-parameter adaptation \cite{li2017stochastic}, and optimization of specific network classes such as discrete weights \cite{li2018optimal}.

In terms of OCP-inspired training algorithms, , two closely-related studies, namely Li et al. \cite{li2017maximum} and Jin et al. \cite{jin2020pontryagin}, draw inspiration from the Pontryagin maximum principle (PMP) \cite{pontryagin2018mathematical}.
Notably, PMP characterizes the first-order necessary condition to the dynamic programming (DP) principle.
Our approach, relying on DP, thus provides a more suitable mathematical framework conducive to the development of higher-order methods.
While sharing a similar differential framework, Jin et al. \cite{jin2020pontryagin} focused on control tasks in a significantly lower dimension, in contrast to our work on training high-dimensional DNNs.
Conversely, Li et al. \cite{li2017maximum} introduced layer-wise objectives as an \emph{augmented} Hamiltonian.
This necessitates additional hyper-parameter tuning and deviates from the algorithmic connections established in our approach, leading to observed instability, as we will show in Sec.~\ref{sec:5.A}.

\section{Backpropagation Revisited: An Optimal Control Viewpoint}

\noindent
\textbf{Notation:}
Given an indexed function $f_s(x_s, \theta_s)$,
we shorthand its derivatives at some fixed point $(\bar x_s, \bar \theta_s)$ by, \eg
$\fracpartial{f_s}{x_s} \equiv f_s^x$
and $\fracpartial{}{\theta_s}f_s^x \equiv f_s^{x\theta}$.
The subscript $s$ can be either discrete $k \in [K]$ or continuous $t \in [0,T]$ index.
Let $A = U \Sigma V^\T$ be singular-value decomposition of a matrix $A$, where $\Sigma = \diag(\lambda_1, \cdots, \lambda_r, 0, \cdots)$. Its pseudo inversion $A^\dagger$, along with the Kronecker $\otimes$ and Hadamard $\odot$ products, are defined by
\begin{align*}
    \textstyle
    &A^\dagger = U \Sigma^\dagger V^\T, \quad
    \Sigma^\dagger = \diag(\frac{1}{\lambda_1}, \cdots, \frac{1}{\lambda_r}, 0, \cdots), \\
  \forall a,b &\in \R^m, ~~
  a \otimes b  = \vectorize(b a^\T), ~~ c = a \odot b \Leftrightarrow c_{i} = a_i b_i,
\end{align*}
where $\vectorize(\cdot)$ denotes matrix vectorization.
We preserve $I$ for the identity matrix and represent all vectors as column vectors.

Let us revisit training of a feedforward neural network
(\eg fully-connected or convolutional nets) with depth $K$:
\begin{equation}
  \label{ff}
  \min_{ \theta_k: k \in [K] } \calL(x_K), \quad \text{s.t. }
  x_{k+1} = f_k(x_k, \theta_k),
\end{equation}
where $x_k$ and $\theta_k$
are the vectorized hidden state and parameter at each layer $k$.
For instance,
$\theta_k$ gathers the weight $W_k$ and bias $b_k$ for a fully-connected layer; then, the propagation rule $f_k$ reads
$f_k(x_k,\theta_k) \coloneqq \sigma(W_k x_k + b_k)$ with $\sigma(\cdot)$ being an activation function.
Given an objective function $\calL(x_K)$, for instance, cross-entropy in classification,
the gradient of the parameter at each layer can be computed by
$\fracpartial{\calL}{\theta_k} = \fu^\T \fracpartial{\calL}{x_{k+1}}$,
where $\fracpartial{\calL}{x_{k}}$ can be computed via the Backpropagation:
\begin{equation}
  \label{bp}
  \fracpartial{\calL}{x_k}
  = \fracpartial{x_{k+1}}{x_k}^\T \fracpartial{\calL}{x_{k+1}}
  = \fx^\T \fracpartial{\calL}{x_{k+1}}.
\end{equation}

The chained constraint in \eqref{ff} can be interpreted as a discrete-time dynamical system,
propagating the \emph{state variable} $x_k$ given the \emph{control variable} $\theta_k$.
With that, the training process,
\ie finding optimal parameters $\{\theta_k\}_{k \in [K]}$ for all layers, can be described by Optimal Control Programming (OCP).
OCP has proven versatile in diverse scientific areas, ranging from
trajectory optimization \cite{theodorou2010generalized},
behavioral understanding \cite{liu2022deep},
and machine learning \cite{chen2022likelihood}.
Mathematical understanding of OCP can be dated back to 1950s, during which its necessary condition was developed by Richard Bellman \cite{bellman1954theory}:
\begin{equation}
  \begin{split}
  \label{bellman}
  &V_k(x_k) = \min_{\theta_k} Q_k(x_k,\theta_k), \quad V_K(x_K) = \calL(x_K) \\
  &\quad~~ Q_k(x_k,\theta_k) \coloneqq \ell(\theta_k) + V_{k+1}(f_k(x_k, \theta_k)).
\end{split}
\end{equation}

\Eqref{bellman} is known as the Bellman equation or the dynamic programming principle.
It decomposes the optimization of a sequence of variables, $\{ \theta_k\}_{k \in [K]}$,
into a sequence of optimization over each of them.
The value function $V_k(\cdot)$ summarizes the optimal cost-to-go for each stage $k$, while the parameter regularization $\ell_k(\cdot)$ is often set as the $\ell_2$-norm weight decay,
$\ell(\theta_k) \coloneqq \frac{\gamma}{2}\norm{\theta_k}^2$ with $\gamma {\ge} 0$,
for training DNNs.

The crux of introducing \eqref{bellman} alongside \eqref{bp} is fundamentally motivated by the shared characteristic that the optimality condition for a chained constraint, exemplified in \eqref{ff}, is commonly formulated through a backward process. This prompts a natural question regarding the relationship between the two, to which we provide elucidation in the subsequent theorem.
\begin{restatable}{thm}{thmone}\label{thm:1}
  Let $(\bar x_k, \bar \theta_k)$ be the solution to the chain constraint in \eqref{ff}.
  Expand the Bellman objective $Q_k(x_k, \theta_k)$ by
  \begin{equation}\label{bp-Q}
    \hspace{-6pt}
    Q_k(\bar x_k{,} \bar \theta_k) {+}~ {\Qx^\T}{(x_k {-} \bar x_k)} {+}~ \Qu^\T(\theta_k {-} \bar \theta_k) {+} {\frac{1}{2}}\norm{\theta_k {-} \bar \theta_k}^2%
    \hspace{-3pt}
  \end{equation}
  where %
  $\Qx, \Qu$ are the first-order derivatives of $Q_k$ w.r.t. $x_k, \theta_k$,
  \begin{equation}
    \label{Qx-Qu}
    \Qx = \fx^\T \Vx  , \quad \Qu = \lu + \fu^\T \Vx.
  \end{equation}
  Then, solving the dynamic programming \eqref{bellman} recovers Backpropagation \eqref{bp} and gradient descent and weight decay $\gamma \ge 0$.
  Specifically, it holds for all $k \in [K]$ that
  \begin{equation*}
    V_k^x = \fracpartial{\calL}{x_k} \quad \text{and} \quad
    \theta_k^\star = \bar \theta_k - \pr{\fracpartial{\calL}{\theta_k} + \gamma \bar \theta_k}.
  \end{equation*}
\end{restatable}
Theorem~\ref{thm:1} unveils the variational structure inherent in the Backpropagation \eqref{bp}, where the backward process instantiates an approximate dynamic programming algorithm. Specifically, it approximates the Bellman objective $Q_k(x_k, \theta_k)$ in \eqref{bellman} using first-order expansions and a quadratic term in $\theta$, resulting in first-order updates.
This solidifies OCP as a mathematical tool for studying neural network training
and, as shown below, unlocks a new class of optimization methods with rich algorithmic advancements.
The proof is provided in Appendix~\ref{sec:proof}.

\section{Optimal Control Theoretic Neural Optimizer}

\noindent
Motivated by Theorem~\ref{thm:1}, in this section we outline a new class of optimizers that capitalizes on the inherent connection with optimal control theory.
All proofs are deferred to Appendix~\ref{sec:proof}.

\subsection{Differential Dynamic Programming (DDP)} \label{sec:ddp}

\noindent
Theorem~\ref{thm:1} suggests that Backpropagation \eqref{bp} expands the Bellman objective $Q_k(x_k, \theta_k)$ up to the first order as in \eqref{bp-Q}.
One may consider second-order approximation on the same fixed point $(\bar x_k, \bar \theta_k)$, where the Bellman equation \eqref{bellman} becomes
\begin{equation}
  \label{Q-full-expand}
  Q_k(\bar x_k, \bar \theta_k) +
  \frac{1}{2}
  \begin{bmatrix}
    1 \\
    \dxt \\
    \dut
  \end{bmatrix}^\T%
  \begin{bmatrix}
    0 & \Qx^\T & \Qu^\T \\
    \Qx & \Qxx & \Qxu \\
    \Qu & \Qux & \Quu
  \end{bmatrix}
  \begin{bmatrix}
    1 \\
    \dxt \\
    \dut
  \end{bmatrix}.
\end{equation}
We shorthand
$\delta x_k \coloneqq x_k - \bar x_k$ and
$\delta \theta_k \coloneqq \theta_k - \bar \theta_k$ and call them the \emph{differential} variables.
The second-order derivatives in \eqref{Q-full-expand} can be computed from \eqref{Qx-Qu} via standard chain rule:
\begin{equation}
  \label{Q2nd}
  \begin{split}
  \Qxx &= \fx^\T \Vxx \fx + \fxx\Vx, \\
  \Quu &= \fu^\T \Vxx \fu + \fuu\Vx + \luu I,\\
  \Qxu &= \fx^\T \Vxx \fu + \fxu\Vx,\quad  \Qux = \Qxu^\T, \\
\end{split}
\end{equation}
where the second-order derivatives of $f_k$ are typically omitted, following a conventional treatment for improved stability \cite{todorov2005generalized}.

Solving
\eqref{Q-full-expand} w.r.t. $\dut$ is a standard Quadratic Programming problem. 
For a sufficiently large $\gamma$ such that the Hessian $\Quu$ is definite,
the analytic solution is known as a \textit{mapping} $\delta \theta_k^\star(\cdot)$,
\begin{equation}
  \label{opt-u}
  \delta \theta_k^\star(\delta x_k) = - \pr{\Quu}^\Inv \pr{ \Qu + \Qux \dxt  },
\end{equation}
composing an open-loop gain from $\Qu$ and a closed-loop gain from $\Qux \dxt$.
That implies that the optimal parameter update rule \eqref{opt-u} functions as a \emph{feedback policy}, adjusting its outputs based on the differential state $\dxt \coloneqq x_k - \bar x_k$.
Conceptually, $\dxt$ can be any deviation away from the fixed point $\bar x_k$ where $Q_k$ is expanded in \eqref{Q-full-expand}.
Standard practices in OCP \cite{tassa2012synthesis,sun2018min} typically involve setting it to the difference when the parameter updates are applied up to layer $k$, starting from $\delta x_0 \coloneqq 0$,
\begin{equation}
  \label{dxk}
  \delta x_{k+1} \coloneqq f_k(\bar x_k + \dxt, \bar \theta_k + \delta \theta_k^\star(\dxt)) - \bar x_{k+1}.
\end{equation}
Importantly, the differential state $\delta x_k$ in \eqref{dxk} encompasses all changes propagated from the preceding layers, including those that potentially induce instability.
The resulting feedback gain therefore enhances 
robustness of the optimization process \cite{de1988differential}.

Substituting the analytic second-order solution in \eqref{opt-u} to \eqref{bellman} yields the second-order expansion of value function 
$V_k(x_k) = Q_k(x_k, \theta_k^\star)$, whose first and second-order derivatives follow
\begin{subequations}
  \label{Vx-xx}
  \begin{align}
    V_k^x &= \Qx - \Qxu \pr{\Quu}^\Inv \Qu, \label{Vx} \\
    V_k^{xx} &= \Qxx - \Qxu \pr{\Quu}^\Inv \Qux. \label{Vxx}
  \end{align}
\end{subequations}
These derivatives \eqref{Vx-xx} suffice to compute the derivatives of the Bellman objective for the preceding layer $Q_{k-1}$, as implied in (\ref{Qx-Qu},\ref{Q2nd}), and its corresponding $\delta \theta_{k-1}^\star$.
As such, it is practically sufficient to propagate \eqref{Vx-xx} rather than the value function itself.

The aforementioned procedure is known as the Differential Dynamic Programming (DDP) algorithm \cite{jacobson1970differential}, an iterative method that, akin to Theorem~\ref{thm:1}, approximates dynamic programming computations except with a second-order expansion.
Consequently, Theorem~\ref{thm:1} implies the following result.%
\begin{restatable}{corollary}{corotwo}\label{coro:2}
  When $\Qux \coloneqq 0$ for all $k$, the DDP algorithm degenerates to the Newton's method, and, if additionally
  $\Quu \coloneqq I$,
  it further recovers gradient descent.
\end{restatable}

\subsection{Outer-Product Factorization}

\noindent
While the DDP algorithm is theoretically applicable for training DNNs, its practical use faces challenges due to the computational complexity in propagating the second-order matrices in (\ref{Q2nd},\ref{Vxx}), particularly concerning the high dimensionality of the parameter $\theta_k$. Recognizing that these computations depend only on the terminal Hessian $V_K^{xx} \coloneqq 
\fracpartial{^2\calL}{x_K\partial x_K}$, \ie the Hessian of $\calL$, we explore a low-rank approximation of this matrix. The subsequent result illustrates how this low-rank factorization persists during backward propagation.
\begin{restatable}{proposition}{propthree}\label{prop:3}
  Let the terminal Hessian be approximated by $V_K^{xx} \approx yy^\T$.
  Then, it holds for all $k$ up to second-order that
  \begin{equation}
    \label{Qxu-fac}
    \Qxx = q_k q_k^\T, \quad \Qxu = q_k p_k^\T,
  \end{equation}
  where $q_k, p_k$ can be computed backward from $r_K = y$,
  \begin{equation}
    \label{qpr}
    q_k = \fx^\T r_{k+1}, ~ p_k = \fu^\T r_{k+1},
    ~ r_k = \sqrt{ 1 {-} p_k^\T \pr{\Quu}^{\dagger} p_k  } q_k.
  \end{equation}
\end{restatable}
Proposition~\ref{prop:3} indicates that the outer-product factorization applied at the final layer is amenable to backward propagation through all preceding layers.
Consequently, in comparison to the original computations (\ref{Q2nd},\ref{Vxx}), efficient computation of second-order matrices, such as $\Qxx$, $\Qxu$, $V^{xx}_k$, and even the feedback update \eqref{opt-u}, can be achieved through auto-differentiation and vector-Jacobian products.
It is noteworthy that although extending to higher-rank factorization is feasible and discussed in Appendix \ref{sec:prop3-ext}, the practical utility of low-rank structures is evident from empirical observations \cite{nar2019cross, lezama2018ole}. We also find empirically that this approximation performs sufficiently well and is easily applicable at scale.

While Proposition~\ref{prop:3} can be applied with any $y$, in practice, we adopt the Gauss-Newton approximation, $y \coloneqq \beta \fracpartial{}{x_K}\calL(x_K)$ with a tunable Gauss-Newton factor $\beta \in (0,1]$, for its relation to the Fisher information matrix and practical successes of natural gradient methods in second-order training of DNNs \cite{botev2017practical,george2018fast,martens2014new}.
This further reduces the computation:
\begin{restatable}{corollary}{corofour}\label{coro:4}
  When $y \coloneqq \beta \fracpartial{}{x_K}\calL(x_K)$ and $\gamma {\coloneqq} 0$, it holds that
  \begin{equation*}
    V_k^x = \prod_{n = k}^{K-1} \pr{\frac{\alpha_{n}}{\alpha_{n+1}}}^2  \fracpartial{\calL}{x_k},
    \quad
    V_k^{xx} = \alpha_k^2 \fracpartial{\calL}{x_k}\fracpartial{\calL}{x_k}^\T,
  \end{equation*}
  where $\alpha_k \in \R$ can be computed backward from $\alpha_K = \beta$,
  \begin{equation*}
    \alpha_k = \sqrt{1 - \alpha_{k+1}^2 \fracpartial{\calL}{x_{k+1}}^\T \fu \pr{\Quu}^\dagger \fu^\T \fracpartial{\calL}{x_{k+1}}  } ~ \alpha_{k+1}.
  \end{equation*}
\end{restatable}
Corollary~\ref{coro:4} reveals another intriguing connection between the DDP algorithm and Backpropagation, wherein the derivatives of the value function $V_k^x$ and $V_k^{xx}$ are effectively the rescaled gradients $\frac{\partial \mathcal{L}}{\partial x_k}$.
Computation of the rescaling factor $\alpha_k$ is non-trivial, relying on the gradient from the subsequent layer $\fracpartial{\calL}{x_{k+1}}$ and the network architecture via $\fu \pr{\Quu}^\dagger \fu^\T$.
Notably, unlike Theorem~\ref{thm:1} and Corollary~\ref{coro:2}, all second-order matrices persist in Corollary~\ref{coro:4}; hence the parameters are updated by layer-wise feedback policies \eqref{opt-u}.
From a practical standpoint, the Gauss-Newton approximation simplifies the propagation down to a single scalar, leading to a significant reduction in complexity.

\subsection{Curvature Approximation} \label{sec:4.C}

\noindent
Our remaining challenge amounts to the computation of $\Quu$ and its inversion $\Quu^\Inv$.
Due to the high dimensionality of $\theta_k$ in over-parametrized models such as DNNs,
these large second-order matrices are computationally intractable to solve, thus necessitating approximation.
As demonstrated in Corollary~\ref{coro:2}, a natural choice can be the identity matrix, $\Quu \approx I$.
Importantly, since the feedback gain remains nontrivial, the resulting algorithm does not degenerate to gradient descent.
Instead, Corollary~\ref{coro:2} suggests that it should be regarded as an optimal control theoretic extension to gradient descent.
Similarly, we can also consider diagonal matrices, which then resembles adaptive curvature approximation imposed in methods such as RMSprop \cite{hinton2012neural} and Adam \cite{kingma2015adam},
\begin{equation}
  \label{adap}
  \Quu \approx \diag\pr{\sqrt{\E[\Qu \odot \Qu]} + \epsilon}.
\end{equation}
where $\epsilon > 0$ is a hyper-parameter for numerical stability.

\begin{table}[!t]

\caption{
  Comparison of parameter update rule,
  $\theta_k \leftarrow \bar \theta_k - \eta \cdot M_k^{-1} d_k$,
  among different optimizers. For brevity, we shorthand 
  $\calL^\theta_k \coloneqq  \nicefracpartial{\calL}{\theta_k}$ and
  $\calL^h_{k} \coloneqq \pr{\nicefracpartial{\sigma}{h_k}}^\T \nicefracpartial{\calL}{x_{k+1}}$.
  Notice how our OCNOpt shares similar curvature structures with prior optimizers while adopting the DDP update directions.
  \label{table:update-rule}
}
\centering
\begin{tabular}{r|cc}
  \toprule
  Optimizers & Curvature Approximation $M_k$ & Update $d_k$ \\
  \midrule
  SGD       & $I$
            & $\calL^\theta_k$ \\
  RMSprop   & $\diag\pr{\sqrt{\E[\calL^\theta_k \odot \calL^\theta_k]} +\epsilon}$
            & $\calL^\theta_k$ \\
  KFAC      & $\E{[x_k x_k^\T]} \otimes \E{[\calL^h_{k} \calL^{h~\T}_{k}]} $
            & $\calL^\theta_k$ \\
  \midrule
  DDP & $\Quu$ in \eqref{Q2nd}
            & $\Qu+\Qux\dxt$ \\
  \midrule
  \specialcellr[]{\textbf{OCNOpt} \\ \textbf{(ours)}} &
            $\left\{\begin{array}{c}
              {I }, \\
              {\diag\pr{\sqrt{\E[\Qu \odot \Qu]} +\epsilon}},\\
              {\E\br{x_k x_k^\T} \otimes \E\br{V^h_k V^{h~\T}_k}}\end{array}\right\}$
            & $\Qu+\Qux\dxt$ \\
  \bottomrule
\end{tabular}
\end{table}

Regarding second-order curvature approximation, a prevalent option in the context of training DNNs pertains to Fisher information or, equivalently, Gauss-Newton matrices.
\linebreak[4]
Popular second-order methods, \eg \cite{martens2015optimizing,grosse2016kronecker}, leverage efficient Jacobian-vector computation for fully-connected layers
$f_k(x_k, \theta_k) {=} \sigma(h_k) {=} \sigma(W_k x_k {+} b_k)$, simplifying, \eg  \eqref{Qx-Qu}, to
\begin{equation*}
\textstyle
  \fu^\T \Vx = \pr{\fracpartial{\sigma}{h_k} \fracpartial{h_k}{\theta_k}}^\T \Vx 
  = x_k \otimes V^h_k,
  \,\, V^h_k := {\fracpartial{\sigma}{h_k}}^\T \Vx.
\end{equation*}
This leads to an efficient factorization of the Gauss-Newton matrix
(assuming $\gamma \coloneqq 0$ for now):
\begin{align*}
  \E\br{ \Qu \Qu^\T} 
  &= \E\br{ \pr{x_k \otimes V^h_k} \pr{x_k \otimes V^h_k}^\T } \\
  &= \E\br{ \pr{x_k x_k^\T} \otimes \pr{V^h_k {V^h_k}^\T} } \\
  &\approx \E\br{x_k x_k^\T} \otimes \E\br{V^h_k {V^h_k}^\T}, \numberthis \label{kfac}
\end{align*}
where the expectation is taken w.r.t. the mini-batch and we invoke the mixed-product property of the Kronecker product in the second line.
The Kronecker factorization in \eqref{kfac} enables an efficient computation for the preconditioned matrix $\Quu^\dagger$ by applying singular-value decompositions, $\E[x_k x_k^\T] = U_1 \Sigma_1 U_1^\top$ and
$\E[V^h_k {V^h_k}^\T] = U_2 \Sigma_2 U_2^\top$, and then computing the pseudo inversion \cite{george2018fast} with $\gamma \ge 0$ via 
(see Appendix~\ref{sec:approx-explain} for details):
\begin{align*}
    \pr{\Quu}^\dagger
    &\approx \textstyle \frac{1}{\gamma} I + \pr{\E\br{x_k x_k^\T} \otimes \E\br{V^h_k {V^h_k}^\T}}^\dagger \\
     &= (U_1 \otimes U_2) (\textstyle \frac{1}{\gamma} I + \Sigma_1^\dagger \otimes \Sigma_2^\dagger) (U_1^\T \otimes U_2^\T). \numberthis \label{kfac-u}
\end{align*}
Effectively, we break down the large matrix inversion into sequence of Kronecker-based computation on much smaller matrices.
The preconditioned update $\pr{\Quu}^\dagger \Qu$ can then be efficiently computed by invoking the formula: $(A \otimes B) \vectorize (C) = \vectorize (B C A^\T)$ for matrices $A,B,C$ of proper sizes.

To summarize, our proposed algorithm, termed Optimal Control Theoretic Neural Optimizer (OCNOpt), shares a similar algorithmic backbone with the DDP algorithm but undergoes substantial modifications to address practical training considerations.
By adopting a curvature approximation based on successful methods from other optimizers, OCNOpt not only benefits from the wisdom of proven techniques but, as indicated by Corollary~\ref{coro:2}, should also be viewed as an optimal control theoretic extension to these methods. Table~\ref{table:update-rule} highlights the distinctions between various classes of optimizers.

\subsection{Extension to General Network Architectures} \label{sec:4.D}

\noindent
Our derivation so far has been tailored to the feedforward propagation rule \eqref{ff}, which, despite its prevalence, is merely one of many alternatives. Considering that modern network architectures often entail intricate dependencies between layers, establishing a general recipe for extending the aforementioned optimal control theoretic optimizers to accommodate modern architectures is a crucial step to enhance their applicability.

Oftentimes, such generalization involves redefining the propagation rule \eqref{ff} and, from which, constructing proper Bellman objectives such that their first-order expansions align with the Backpropagation processes, similar to Theorem~\ref{thm:1}.
Take the residual block in Fig.~\ref{fig:resnet} for instance.
The skip connection can be made Markovian as in \eqref{ff} via $x_k \coloneqq [z_k, z_{k-1}]^\T$, $x_{k+1} := [z_{k+1}, z_{k,\text{res}}]^\T$, and (see Appendix \ref{sec:res-app} for details)
\begin{align}
  \label{F}
  x_{k+1} = F_k(x_k, \theta_k, \phi_k) 
  =
  \begin{bmatrix}
    f_k(z_k, \theta_k) \\
    f_k(z_{k-1}, \phi_k)
  \end{bmatrix}.
\end{align}
This leads to a Bellman objective with an additional input $\phi_k$,
\begin{align*}
  Q_k(x_k, \theta_k, \phi_k) \coloneqq \ell(\theta_k) + \ell(\phi_k) + V_{k+1}(F_k(x_k, \theta_k, \phi_k)),
\end{align*}
which coincides with the optimality condition appearing in \emph{dynamic games}---a discipline of interactive decision-making building upon optimal control and game theory \cite{yeung2006cooperative,petrosjan2005cooperative}.
In this vein, each layer acts like a player in a dynamic game connected through the network propagation, and
the (p)layers interact collectively to solve the same end goal $\calL(x_K)$.

\begin{figure}[!t]
  \centering
  \includegraphics[width=.9\linewidth]{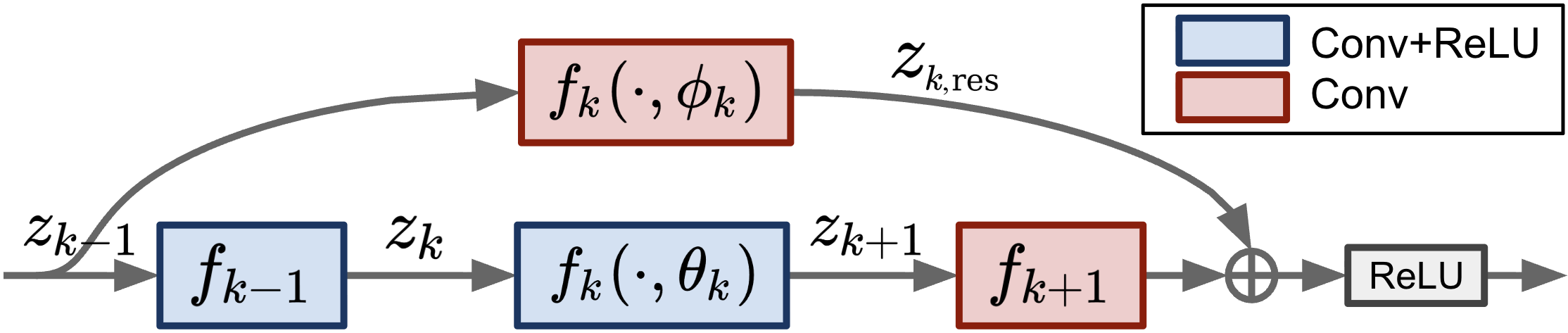}
  \caption{Example of a residual block as $F_k(x_k, \theta_k, \phi_k)$ in \eqref{F}.}
  \label{fig:resnet}
\end{figure}

Algorithmically, similar parameter feedback updates can be derived by either freezing $\bar\phi_k$ and expanding $Q_k(x_k, \theta_k; \bar\phi_k)$ as in \eqref{Q-full-expand}, or redefining $\tilde \theta_k \coloneqq [\theta_k, \phi_k]$ and instead expanding $Q_k(x_k, \tilde \theta_k)$. The expansions correspond respectively to non-cooperative and cooperative dynamic games, showcasing the rich algorithmic opportunities introduced by the game-theoretic perspective.
It is noteworthy that the construction of \eqref{F} is \emph{not} unique, as the skip connection layer $f(\cdot, \phi)$ in Fig.~\ref{fig:resnet} may also be placed to construct $F_{k-1}$ or $F_{k+1}$. We will discuss its game-theoretic implications in Sec.~\ref{sec:game}.

When the skip connection layer is an identity map, \eqref{F} can be rewritten as $z_{k+2} = z_{k-1} + \epsilon \cdot (f_{k+1} \circ f_k \circ f_{k-1})(z_{k-1})$ with $\epsilon = 1$. 
In the continuous-time limit ($\epsilon \rightarrow 0$), this leads to a crucial extension 
where the propagation dynamics become a Neural Ordinary Differential Equation (ODE) \cite{chen2018neural}:
\begin{align}
  \label{node}
  \fracdiff{}{t} x_t = F(t, x_t, \theta),
\end{align}
where $t \in [0,T]$ is the continuous time index.
Typically, $F(t, x_t, \theta)$ is represented by a feedforward network
with $(t,x_t)$ as its inputs, and $\theta$ gathering parameters across all layers.
Computing the gradients $\fracpartial{\calL(x_T)}{\theta}$
involves solving the continuous-time analogy of the Backpropagation process \eqref{bp},
well-known in OCP literature \cite{pontryagin2018mathematical} as the adjoint equation,
\begin{align}
  \label{bp-cont}
  - \fracdiff{}{t} \pr{\fracpartial{\calL}{x_t}} = {\fracpartial{F}{x_t}}^\T \fracpartial{\calL}{x_t}.
\end{align}

\begin{algorithm}[t]
    \caption{\textbf{O}ptimal \textbf{C}ontrol Theoretic \textbf{N}eural \textbf{Opt}imizer}
    \label{alg}
    \begin{algorithmic}[1]
        \REQUIRE Deep neural network (DNN),
                Gauss-Newton factor $\beta \in (0,1]$,
                learning rate $\eta > 0$
        \REPEAT
            \STATE Forward propagate $x_k, \forall k \in [K]$, via \eqref{ff} or \eqref{F}
            \STATE Set $V_K^x \coloneqq \fracpartial{\calL}{x_K}$ and $r_K \coloneqq \beta \fracpartial{\calL}{x_K}$
            \FOR{$k=K-1$ {\bfseries to} $0$}
                \STATE Compute $\Qx, \Qu, q_k, p_k$ via (\ref{Qx-Qu},\ref{qpr})
                \STATE Estimate curvature $\Quu$ with any option in Table~\ref{table:update-rule}
                \STATE Invoke (\ref{opt-u},\ref{Qxu-fac}) and compute the feedback policy $\delta \theta_k^\star$
                \STATE Compute $V_k^x$ and $r_k$ via (\ref{Vx},\ref{qpr})
            \ENDFOR
            \FOR{$k=0$ {\bfseries to} $k-1$}
                \STATE Compute the differential state $\dxt$ via \eqref{dxk}
                \STATE Update layer parameter $\theta_k \leftarrow \bar \theta_k + \eta \cdot \delta \theta_k^\star(\dxt)$
            \ENDFOR
        \UNTIL{converges}
    \end{algorithmic}
\end{algorithm}

Equations~(\ref{node},\ref{bp-cont}) model a continuous-time Optimal Control Problem (OCP). The corresponding Bellman equation is formulated as $V_0(x) = \min_\theta Q_0(x, \theta)$, where the continuous-time Bellman objective $Q_t(x_t, \theta)$ is defined as:
\begin{align}
  \label{Q-cont}
 Q_t(x_t, \theta) {\coloneqq} {\int_t^T} \ell(\theta) \dt + \calL\pr{x_t + {\int_t^T} F(\tau, x_\tau,\theta) \rd \tau}.
\end{align}
With \eqref{Q-cont}, the subsequent derivation involves a straightforward process of performing a second-order expansion on $Q_t(x_t, \theta)$, deriving the backward dynamics of the resulting second-order derivatives, and finally employing the Gauss-Newton approximation.
We consolidate the final results below and direct readers to Appendix~\ref{sec:proof} for a detailed derivation.
\begin{restatable}{proposition}{propfive}\label{prop:5}
  Let %
  $Q_T^{xx} \approx \fracpartial{\calL}{x_T} \fracpartial{\calL}{x_T}^\T$
  and $\ell(\theta) \coloneqq \frac{\gamma}{2}\|\theta\|^2$.
  Then, the second-order matrices can be factorized into
  \begin{align}
    \label{Quu-cont}
    Q_0^{xx} = q_0 q_0^\T, ~~ Q_0^{x\theta} = q_0 p_0^\T, ~~
    Q_0^{\theta\theta} = p_0 p_0^\T + \gamma T I,
  \end{align}
  where $q_0, p_0$ can be computed from the following backward ODEs with the boundary condition $(q_T, p_T) \coloneqq (\fracpartial{\calL}{x_T}, 0)$,
  \begin{align}
    \label{qt-cont}
    - \fracdiff{}{t} q_t = {\Fx}^\T q_t, \quad - \fracdiff{}{t} p_t = {\Fu}^\T q_t.
  \end{align}
\end{restatable}
Proposition~\ref{prop:5} is the continuous-time extension of Corollary~\ref{coro:4}. 
In practice, computing the inversion $Q_0^{\theta\theta\Inv}$ is still prohibitively expensive. 
Since $F$ itself is a DNN, $F(t,x_t,\theta) := (f_K \circ \cdots \circ f_1)(t,x_t)$ where $f_k := f_k(\cdot, \theta_k)$, 
we can decompose $p_t = [\cdots, p_{t,k}, \cdots]^\T$ layer-wise. 
The evolution of $p_{t,k}$ follows 
an ODE whose vector field can be factorized similar to \eqref{kfac},
\begin{align*}
    -\fracdiff{}{t} p_{t,k} = {\fracpartial{F}{\theta_k}}^\T q_t = z_{t,k} \otimes h_{t,k}, \quad
    h_{t,k} := \fracpartial{F}{z_{t,k+1}}^\T q_t,
\end{align*}
where $z_{t,k}$ is the input of the layer $f_k$ at time $t$.
We can thus approximate
the layer-wise preconditioning matrix via
\begin{align}\label{eq:ekfac-snopt}
    Q^{\theta\theta}_{0,k} =
    p_{0,k} p_{0,k}^\T 
    &\approx \int^0_T \pr{z_{t,k} z_{t,k}^\T} \dt \otimes \int^0_T \pr{h_{t,k+1} h_{t,k+1}^\T} \dt.
\end{align}
This results in an efficient second-order optimization method for training Neural ODEs,
a task that is otherwise challenging without the introduction of OCP and Bellman principle.

Algorithms~\ref{alg} and \ref{alg:cont} provide pseudo codes summarizing our OCNOpt for training standard DNNs or Neural ODEs. The procedures share a similar backbone with previous optimizers, iteratively performing: 
{(1)} forward propagation of the neural network dynamics,
{(2)} backward computation of the approximate Bellman objective, and
{(3)} layer-wise parameter updates.

\begin{algorithm}[t]
    \caption{\textbf{OCNOpt} for training Neural ODEs}
    \label{alg:cont}
    \begin{algorithmic}[1]
        \REQUIRE Feedforward DNN $F(\cdot,\cdot,\theta)$
                where $\theta \coloneqq \{ \theta_k \}_{k=0}^K$ gathers the layer-wise parameter $\theta_k$,
                learning rate $\eta > 0$ %
        \REPEAT
            \STATE Forward propagate $x_t$ via the forward ODE \eqref{node}
            \STATE Set $(q_T, p_T) \coloneqq (\fracpartial{\calL}{x_T}, 0)$
            \STATE Solve $q_0$ and $p_0$ via the backward ODEs \eqref{qt-cont}
            \STATE Compute layer-wise curvature $Q_{0,k}^{\theta\theta}$ via \eqref{eq:ekfac-snopt}
            \STATE Compute the precondition $\delta \theta_k^\star = - {Q_{0,k}^{\theta\theta}}^\Inv Q_{0,k}^{\theta} $ via \eqref{kfac-u}
            \FOR{$k=0$ {\bfseries to} $k-1$}
                \STATE Update layer parameter $\theta_k \leftarrow \bar \theta_k + \eta \cdot \delta \theta_k^\star$
            \ENDFOR
        \UNTIL{converges}
    \end{algorithmic}
\end{algorithm}

\section{Result}

This section showcases OCNOpt as a competitive optimizer that leverages efficient factorization (Propositions~\ref{prop:3} and~\ref{prop:5}) and curvature approximation (Sec.~\ref{sec:4.C}) for second-order derivatives to achieve computational efficiency and practical training improvements across diverse DNN architectures.
Experimental details and additional results are in Appendices \ref{app:d} and \ref{app:e}.

\subsection{Training DNNs} \label{sec:5.A}

\noindent
We first evaluate the efficacy of our \textbf{OCNOpt} across diverse image classification datasets, chosen for their relevance as the suitable testbeds for a variety of DNN architectures.
The datasets encompass a spectrum of complexities, including WINE, a tabular dataset featuring 3 classes and 13 dimensions, and well-established image classification benchmarks such as DIGITS, MNIST, FMNIST, SVHN, CIFAR10, and CIFAR100.
DIGITS and MNIST datasets comprise grayscale handwritten digits with resolutions of 8$\times$8 and 28$\times$28, respectively.
FMNIST shares the same resolution as MNIST but focuses on fashion-related images.
On the other hand, SVHN, CIFAR10, and CIFAR100 datasets consist of colorful 32$\times$32 images portraying street house numbers and natural objects.
In terms of DNN architectures, our evaluation covers:
\begin{itemize}
  \item Fully-connected networks (FCNs): Comprising 5 fully-connected layers with hidden dimensions ranging from 10 to 32, depending on the dataset size.
  \item Convolutional neural networks (CNNs): Comprising 4 convolution layers (using a 3$\times$3 kernel with 32 channels) followed by 2 fully-connected layers.
  \item Residual-based networks: Consisting of 3 residual blocks, each with a skip connection as illustrated in Fig.~\ref{fig:resnet}. For larger datasets such as CIFAR10 and CIFAR100, we adopt the standard implementation of ResNet18.
  \item Inception-based networks: Featuring an inception block with 4 feature extraction paths of varying kernel sizes. 
  These networks can be likened to a 4-player dynamic game, following our discussion in Sec.~\ref{sec:4.D}.
\end{itemize}

\begin{table*}[!t]
    \caption{Performance comparison on test-time accuracy (\%; higher is better) across various datasets, network architectures, \\ and training optimizers, including standard first and second-order optimizers, OCP-insipred methods, and our \textbf{OCNOpt}.
    \label{tb:ddp-clf}
    }
    \centering
    \begin{tabular}{r!{\vrule width 0.5pt}r!{\vrule width 1pt}cccc!{\vrule width 0.5pt}cc!{\vrule width 0.5pt}c}
    \toprule
    \multirow{2}{*}{Network Type}
    & \multirow{2}{*}{Dataset} & \multicolumn{4}{c|}{Standard Optimizers} & \multicolumn{2}{c|}{OCP Methods} & \multirow{2}{*}{\textbf{ OCNOpt (ours)}} \\ [2pt]
    &                          & SGDm & RMSProp & Adam & EKFAC            & EMSA & vanilla DDP \\
    \midrule
    &{WINE}&
    {94.35} & 98.10 & 98.13 & 94.60 & 93.56 & 98.00 & \textbf{98.18} \\[1pt]
    Fully-
    &{DIGITS}&
    {\textbf{95.36}} & 94.33 & 94.98 & 95.24 & 94.87 & 91.68 & {95.13} \\[1pt]
    Connected
    &{MNIST}&
    {92.65} & 91.89 & 92.54 & 92.73 & 90.24 & 90.42 & \textbf{93.30} \\[1pt]
    &FMNIST&
    {82.49} & 83.87 & 84.36 & 84.12 & 82.04 & 81.98 & \textbf{84.98} \\[1pt]
    \midrule
    &{MNIST}&
    {97.94} & 98.05 & 98.04 & 98.02 & 96.48 & & \textbf{98.09} \\[1pt]
    Convolutional
    &{SVHN}&
    {89.00} & 88.41 & 87.76 & 90.63 & 79.45 & n/a & \textbf{90.70} \\[1pt]
    &{CIFAR10}&
    {71.26} & 70.52 & 70.04 & 71.85 & 61.42 & & \textbf{71.92} \\[1pt]
    \midrule
    \multirow{4}{*}{Residual}
    & {MNIST}     &  98.65  &  98.61 &  98.49  &  \textbf{98.77}  & 98.25 & \multirow{4}{*}{n/a} & {98.76}  \\[1pt]
    & {SVHN}      &  88.58  &  88.96 &  89.20  &  88.75  & 87.40  & & \textbf{89.91} \\[1pt]
    & {CIFAR10}   &  82.94 &  83.75  &  85.66  &  85.65  & {75.60}  & & \textbf{85.85}  \\[1pt]
    & {CIFAR100}  &  71.78  &  71.65 &  71.96  &  71.95  & {62.63} & & \textbf{72.24} \\[1pt]
    \midrule
    & {MNIST}     &  97.96 & 97.75 & 97.72 &  97.90  & 97.39 & & \textbf{98.03}  \\[1pt]
    Inception
    & {SVHN}      &  87.61 &  86.14 &  86.84 &  88.89  & 82.68 & n/a &  \textbf{88.94} \\[1pt]
    & {CIFAR10}   &  76.66 &  74.38 &  75.38 &  77.54  & {70.17} &   & \textbf{77.72}  \\[1pt]
    \bottomrule
    \end{tabular}
\end{table*}

\begin{figure*}[!t]
    \centering
    \includegraphics[width=.95\linewidth]{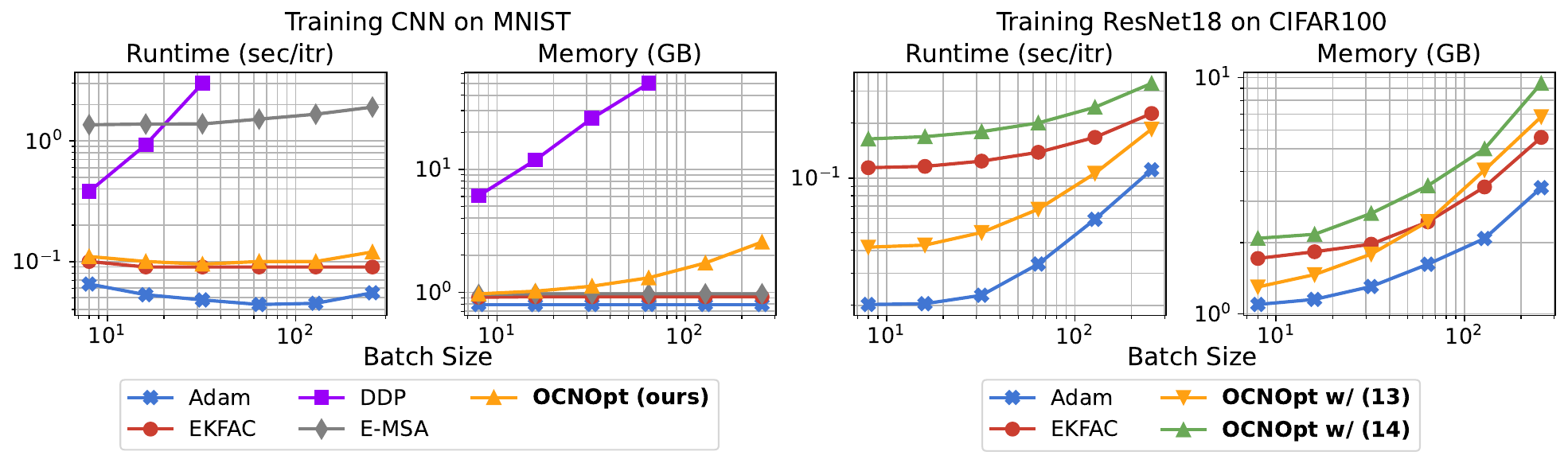}
    \caption{
        Computational complexity between various optimization methods on training (left) a CNN on MNIST dataset and (right) a ResNet18 on CIFAR100 with respect to batch sizes. Notice that all axes use log scale.
    }
    \label{fig:complexity}
\end{figure*}

Our OCNOpt, bridging conventional training methods and the Optimal Control Problem (OCP) principle, is compared against baselines from both domains.
For standard optimizers, we include first-order methods such as SGD (with momentum), RMSprop, and Adam, along with second-order methods such as EKFAC \cite{george2018fast}, sharing a similar computation approach with \eqref{kfac} in eigen-spaces.
For OCP-inspired optimizers, we compare OCNOpt with vanilla DDP (refer to Table~\ref{table:update-rule}) and E-MSA \cite{li2017maximum}, another second-order method except built upon the Pontryagin framework.
The batch size is set to 8-32 for FCNs training and 128 for other architectures.
We implement OCNOpt following Alg.~\ref{alg} and report the best result among various curvature approximations available from Table~\ref{table:update-rule}.
The detailed ablation study is deferred to Sec.~\ref{sec:ablation}.

Table~\ref{tb:ddp-clf} consolidates OCNOpt's performance across various modern DNNs for classification tasks.
Clearly, OCNOpt consistently achieves competitive or superior results compared to standard optimizers across diverse datasets and network architectures.
Importantly, OCNOpt \emph{outperforms} OCP-inspired methods by a significant margin, as these methods often face unstable training dynamics and require meticulous hyper-parameter tuning.
For example, vanilla DDP struggles beyond FCNs' problem sizes.
Despite sharing a similar algorithmic foundation, OCNOpt distinguishes itself from DDP by adapting amortized curvature, which substantially stabilizes the training process.
We highlight this distinction, positioning OCNOpt as a method effectively amalgamating the strengths of both standard and OCP-inspired optimization paradigms.

In Figure~\ref{fig:complexity}, we present the computational complexity in per-iteration runtime and memory consumption w.r.t. different batch sizes.
The results depict the training of CNNs on MNIST and ResNet18 on CIFAR100, chosen to represent small and large networks, respectively.
For smaller networks, OCNOpt exhibits comparable runtime to standard optimizers, with only a minor increase in memory usage attributed to the incorporation of layer-wise feedback policies.
In practice, the differences tend to diminish for larger networks, where the runtime of OCNOpt is approximately $\pm$40\% compared to second-order methods like EKFAC, contingent upon the chosen curvature approximation, \eqref{adap} and \eqref{kfac}.
Notably, OCNOpt consistently outperforms other OCP methods by a large margin, thereby significantly enhancing its practical applicability.

\begin{table*}[!t]
    \caption{
        Evaluation of Neural ODEs across three diverse applications. We report test-time accuracy (\%; higher is better) on image classification and time-series prediction, and negative log-likelihood (NLL; lower is better) on continuous normalizing flow.
        \label{tb:snopt}
    }
    \centering
    \begin{tabular}
        {lccccccccc}
    \toprule
    & \multicolumn{3}{c}{Image Classification} & \multicolumn{3}{c}{Time-Series Prediction} & \multicolumn{3}{c}{Continuous Normalizing Flow} \\[2pt]
    & {MNIST}  & {SVHN} & {CIFAR10}
    & {SpoAD}  & {ArtWR} & {CharT}
    & Circle & {Gas}  & {Miniboone} \\
    \midrule
    Adam          & 98.83          &  91.92          & 77.41          & 94.64          & 84.14          & 93.29
                  & 0.90           & -6.42          &  13.10 \\[2pt]
    SGD           & 98.68          &  93.34          & 76.42          & \textbf{97.70} & 85.82          & 95.93
                  & 0.94           & -4.58          &  13.75 \\[1pt]
    \midrule
    \specialcelll{ \textbf{OCNOpt} \\ (\textbf{ours})}
                  & \textbf{98.99} &  \textbf{95.77} & \textbf{79.11} & 97.41          & \textbf{90.23} & \textbf{96.63}
                  &\textbf{0.86}  & \textbf{-7.55} &  \textbf{12.50} \\
    \bottomrule
    \end{tabular}
\end{table*}

\begin{figure*}[!t]
    \centering
    \subfloat{%
        \includegraphics[width=.7\linewidth]{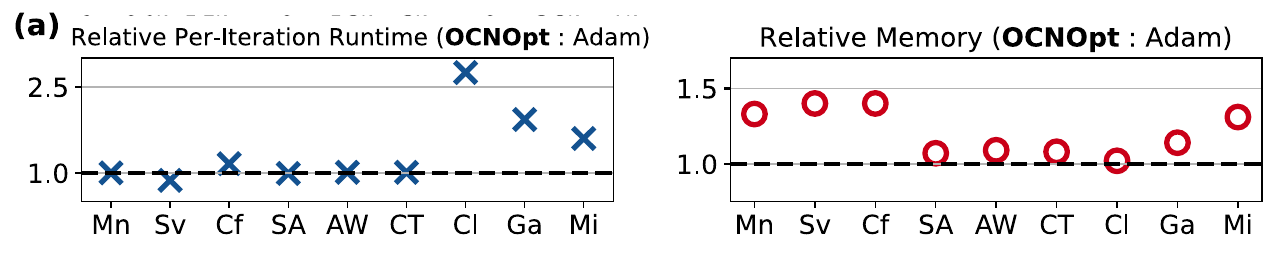}
        \label{fig:convergence-a}
    }\\
    \vskip -0.05in
    \subfloat{%
    \includegraphics[width=.8\linewidth]{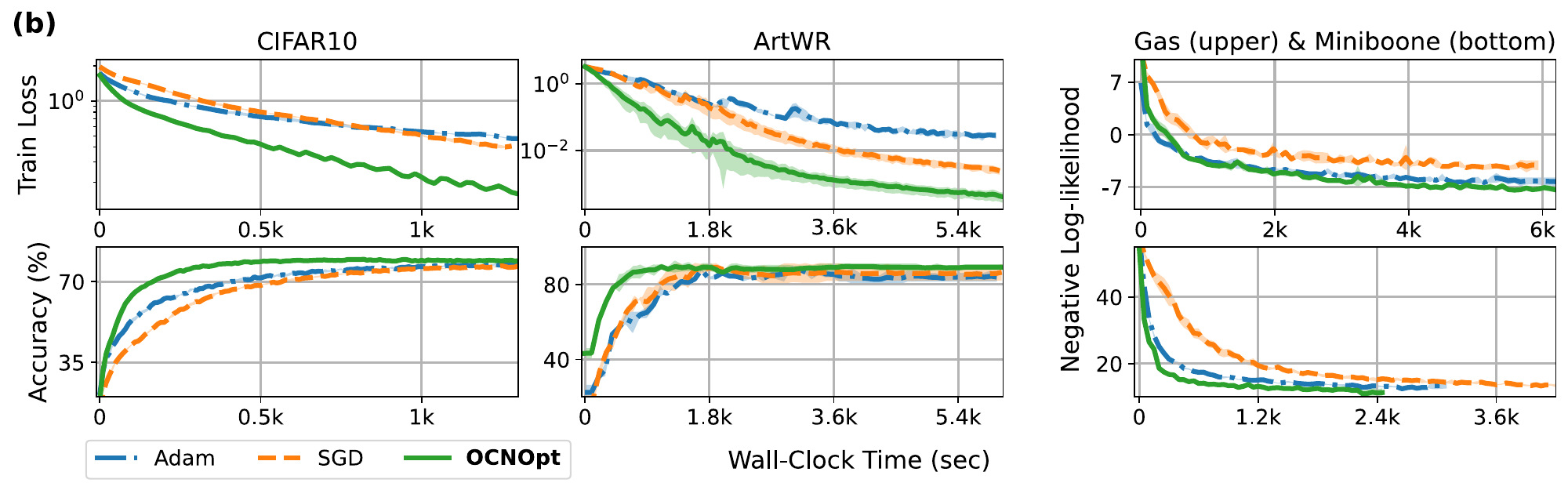}
        \label{fig:convergence-b}%
    }%
    \caption{
        (a) Relative computational complexity of our OCNOpt compared to the first-order method Adam, which are donoted by the black dashed lines. We shorthand, \eg ``Mn'' as MNIST, ``Sv'' as SVHN, and so on.
        (b) Convergence as a function of \emph{wall-clock} runtime between our OCNOpt and other first-order methods, including Adam and SGD.
        Note that the ``k'' in the x-axis abbreviates ``$1000$''.
    }
    \label{fig:convergence}
\end{figure*}

\subsection{Training Neural ODEs} \label{sec:5.B}

\noindent
Next, we proceed to assess OCNOpt's performance in training continuous-time deep architectures, specifically Neural ODEs \cite{chen2018neural}.
The continuous-time nature of these models makes them versatile across various applications beyond classification.
We explore three distinct applications, each comprising three datasets, yielding a total of nine datasets as follows:
\begin{itemize}
  \item Image classification: MNIST, SVHN, CIFAR10, discussed previously in Sec.~\ref{sec:5.A}. All three datasets encompass 10 label classes.
  \item Time-series prediction: Following prior works \cite{kidger2020neural}, we select three datasets from the UEA time series archive \cite{bagnall2018uea}, namely SpokenArabicDigits (SpoAD), Articulary Word Recognition (ArtWR), and Character Trajectories (CharT).
  SpoAD concerns speech data whereas ArtWR and CharT entail data associated with motion.
  \item Continuous normalizing flow (CNF): Involving the transformation of a Gaussian prior into target distributions. We consider the synthetic dataset Circle, reported in \cite{chen2018neural}, and tabular datasets Gas and Miniboone from \cite{grathwohl2019ffjord}.
\end{itemize}

Propagation of the Neural ODEs \eqref{node} involves a numerical integrator, and we employ the widely used Runge-Kutta 4(5) \cite{dormand1980family} on all datasets.
For the actual DNNs $F(t,x_t,\theta)$ in \eqref{node}, we utilize CNNs for image classification and adopt networks from \cite{rubanova2019latent} and \cite{grathwohl2019ffjord} for time-series prediction and CNF, respectively.
The batch size is set to 256, 512, and 1000 for ArtWord, CharTraj, and Gas, respectively, and 128 for the others.
Finally, as there are no second-order methods available for training Neural ODEs, our focus remains primarily on first-order methods as baselines.

Table~\ref{tb:snopt} provides insights into the learning performance for each problem, reporting negative log-likelihood (NLL) for CNF and test-time accuracy (\%) for the remaining datasets.
OCNOpt consistently outperforms standard first-order baselines on almost all datasets, excluding SpoAD.
Notably, these achievements do \emph{not} incur substantial runtime complexity or memory costs, as depicted in Fig.~\ref{fig:convergence}.
On image and time-series datasets, OCNOpt's per-iteration runtime closely aligns with that of Adam, showcasing a significantly faster convergence rate compared to first-order baselines due to its second-order updates.
In the case of CNF, where additional propagation of probability density is required alongside \eqref{node}, OCNOpt is approximately 1.5 to 2.5 times slower but still exhibits faster overall convergence in wall-clock time.
However, due to the extra storage of second-order matrices, OCNOpt generally incurs a 10-40\% increase in memory consumption.
Despite this, the actual rise in memory remains within 1GB for all datasets, thus remaining feasible on standard GPU machines.

\section{Discussion}

\subsection{Ablation Study} \label{sec:ablation}
\begin{figure*}[!t]
  \centering
  \includegraphics[width=.8\linewidth]{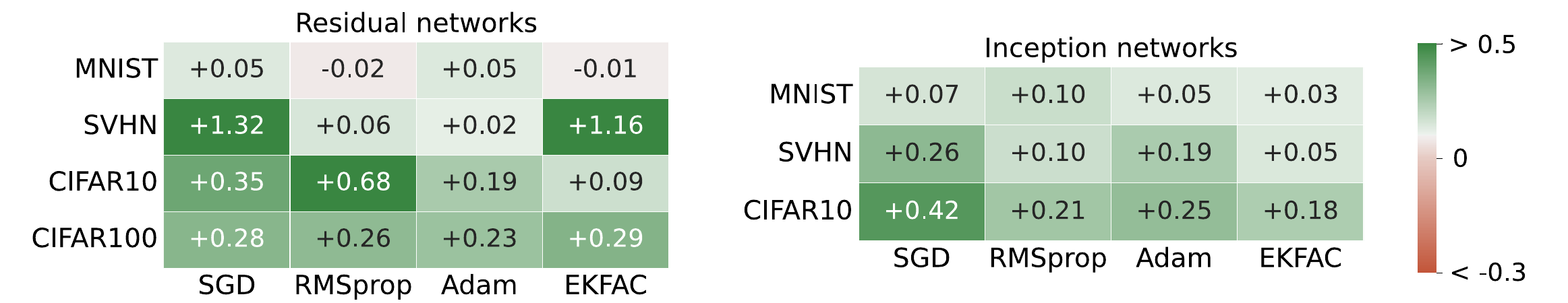}
  \caption{
    Accuracy (\%) improvement ({+}) or degradation ({-}) observed when layer-wise feedback policies are enabled for each best-tuned baseline from Table~\ref{tb:ddp-clf}.
    Color bar is scaled for best view.
  }
  \label{fig:ablation1}
\end{figure*}
\begin{figure}[!t]
  \centering
  \includegraphics[width=.96\linewidth]{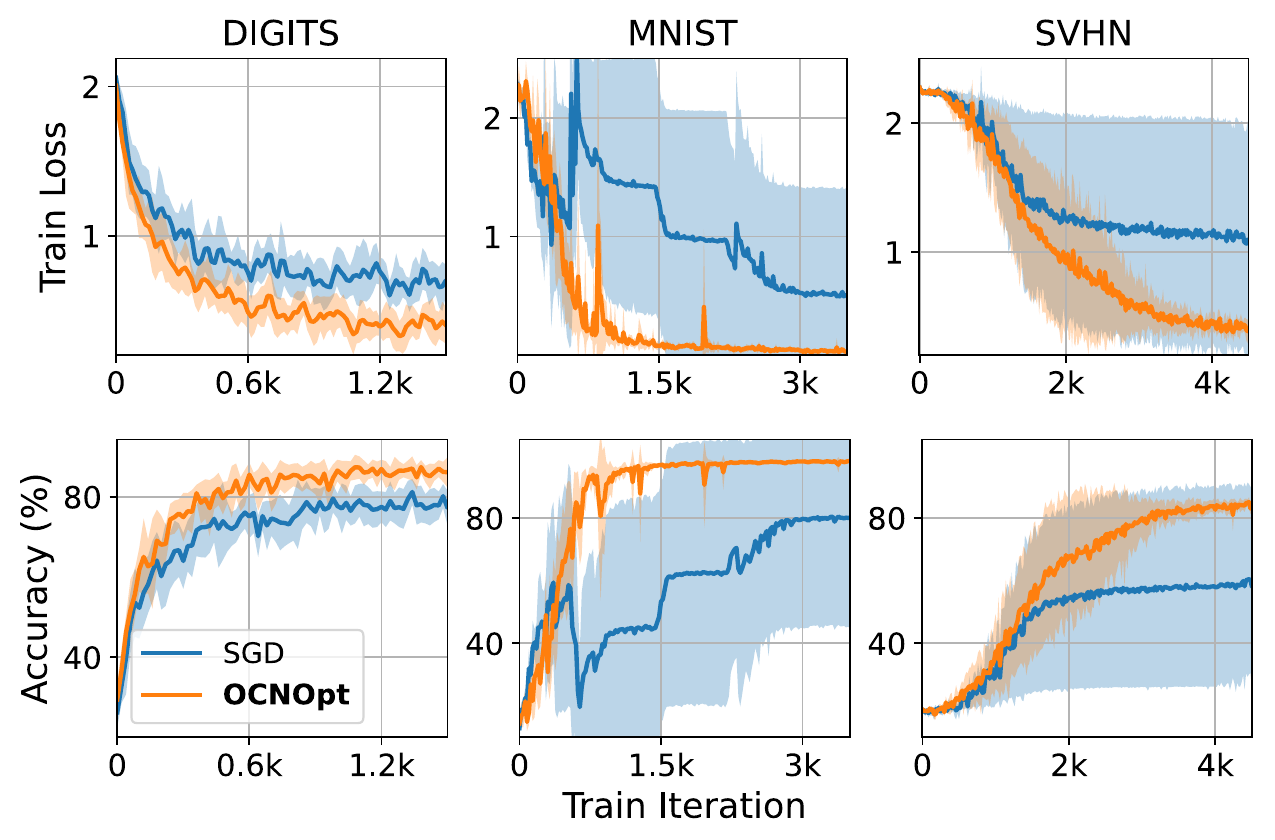}
  \caption{
    Demonstration of the stabilization effect of OCNOpt on training convergence under unstable hyper-parameters.
    The learning rates are set to 0.6, 1.0, and 2.0 for DIGITS, MNIST, and SVHN, respectively.
    We employ FCNs for DIGITS and inception networks for MNIST and SVHN.
  }
  \label{fig:robust}
\end{figure}

OCNOpt distinguishes itself from standard methods on training standard DNNs through
the incorporation of layer-wise feedback policies.
To assess the efficacy of these feedback policies,
we conduct an ablation study using the algorithmic connection established in Corollary~\ref{coro:2}.
Specifically, when $\Qux$ vanishes, OCNOpt reverts to the method associated with the curvature approximation.
For instance, OCNOpt with identity (\emph{resp.} adaptive diagonal \eqref{adap} and Gauss-Newton \eqref{kfac}) curvature approximation produces the same updates as SGD (\emph{resp.} RMSprop and EKFAC) when all $\Qux$ are zeroed out.
In essence, these OCNOpt variants can be regarded as elevating standard methods to accommodate the closed-loop structure,
differing in training procedures only due to the presence of $\Qux$.
This allows them to adjust parameter updates based on changes in the differential state $\dxt$, as depicted in \eqref{opt-u}.

With this insight, we present in Fig.~\ref{fig:ablation1} the performance difference when the baselines in Table~\ref{tb:ddp-clf} are elevated with the closed-loop structure from DDP, focusing specifically on residual and inception networks due to their popularity.
Each grid in the figure corresponds to a distinct training configuration, and the hyper-parameters are kept the same between baselines and their OCNOpt variants.
Therefore, any performance gap solely arises from the feedback policies.
The results indicate that the feedback mechanisms tend to enhance performance or, at the very least, have a neutral impact.
This aligns with prior studies suggesting that feedback improves numerical stability and convergence \cite{liao1992advantages,murray1984differential}, especially when dealing with problems that inherit chained constraints, such as DNNs.

\begin{figure}[!t]
  \centering
  \includegraphics[width=.99\linewidth]{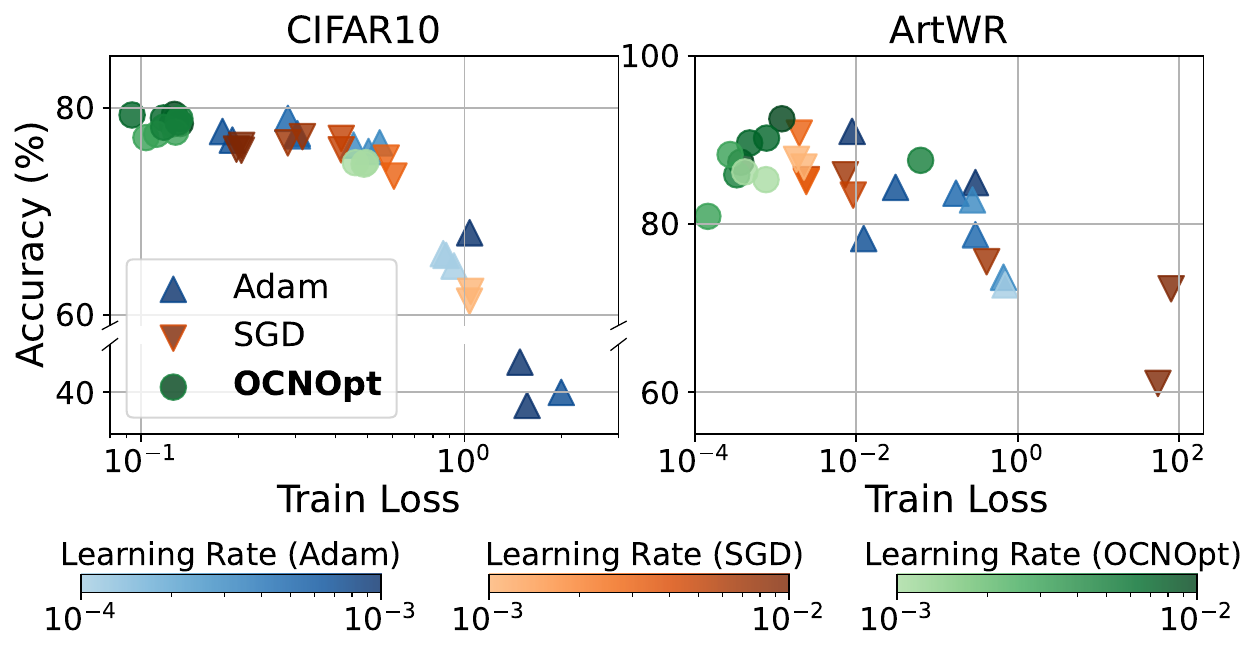}
  \caption{
    Sensitivity analysis where
    each sample represents a training result using different optimizers and learning rates, annotated by distinct symbols and colors.
    Note that $x$-axes are presented in $\log$ scale.
  }
  \label{fig:ablation2}
\end{figure}

From an algorithmic perspective, the impact of feedback is expected to be more significant when a larger step size is employed.
This is because, as shown in \eqref{dxk}, a larger $\delta \theta^*_{k-1}$ increases $\dxt$, subsequently amplifying the feedback $\Qux\dxt$.
This hypothesis is confirmed in Fig.~\ref{fig:robust}, where we consider a relatively large learning rate.
The results clearly illustrate that incorporating feedback updates greatly enhances robustness in training, leading to faster convergence with reduced variance.
While the SGD baseline struggles to make stable progress, OCNOpt converges almost flawlessly.
These improvements shed light on the benefits gained from
bridging optimal control theory and original training methods in DNNs.

As for the training of continuous-time models,
where OCNOpt stands out as an efficient second-order method,
empirical observations indicate that utilizing preconditioned updates significantly mitigates sensitivity to hyper-parameters, such as the learning rate.
This is effectively illustrated in Fig.~\ref{fig:ablation2}, where we sample distinct learning rates from an appropriate interval for each method (represented by different color bars) and record their training results post-convergence.
The results are evident: our method not only converges to higher accuracies with lower training losses, but these values also exhibit a more concentrated distribution on the plots, indicating that OCNOpt achieves superior convergence in a more consistent manner across various hyper-parameter configurations.

\subsection{Game-Theoretic Application} \label{sec:game}

\begin{table}[!t]
  \caption{
     Comparison of EKFAC and OCNOpt with different alignment strategies, reporting accuracy (\%) after 1600 training iterations for SVHN and 6000 for CIFAR10.
  }
  \label{tb:bandit}
  \centering
  \begin{tabular}{l!{\vrule width 0.7pt}cccc}
    \toprule
    \multirow{2}{*}{Dataset} & \multirow{2}{*}{EKFAC} & \multicolumn{3}{c}{OCNOpt + Aligning Strategy} \\ %
     &  & fixed & random & adaptive \\
    \midrule
    SVHN    & 87.49 & 88.20 & 88.12 & \textbf{88.33} \\
    CIFAR10 & 84.67 & 85.20 & 85.27 & \textbf{85.65} \\
    \bottomrule
    \end{tabular}
\end{table}
\begin{figure}[!t]
  \centering
  \subfloat{%
    \includegraphics[width=\linewidth]{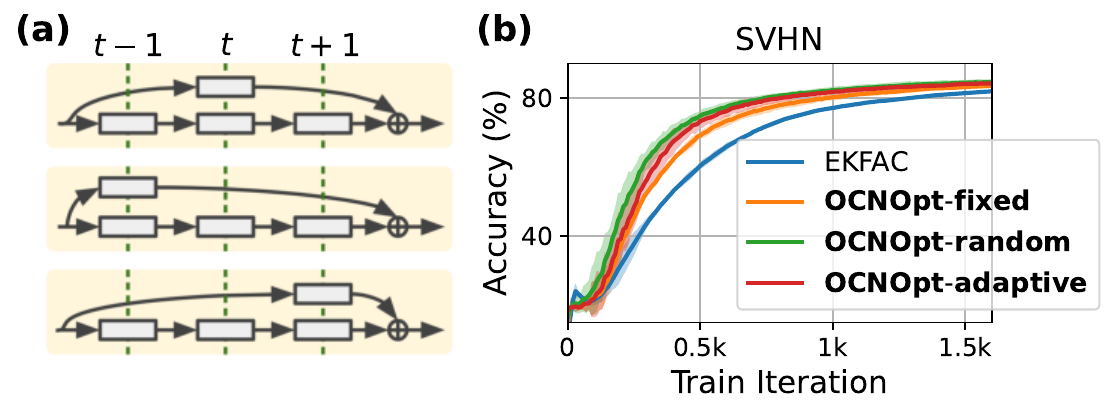}
    \label{fig:bandit-a}
  }%
  \subfloat{%
    \textcolor{white}{\rule{1pt}{1pt}}
    \label{fig:bandit-b}%
  }%
  \caption{
    (a) Various placements of the same residual block result in distinct OCNOpt updates, which are unrecognizable to standard training methods.
    (b) Training convergence w.r.t. different aligning strategies.
  }
  \label{fig:bandit}
\end{figure}

\paragraph*{\textit{(i)} Adaptive alignment using multi-armed bandit}

In Sec.~\ref{sec:4.D}, we briefly mentioned that the construction of the dynamic game \eqref{F} is not unique.
For example, placing the skip connection module in Fig.~\ref{fig:bandit-a} at different locations leads to different definitions of $F_t$, resulting in diverse OCNOpt updates.
It is natural to question the optimal strategy for aligning the layers of the network in our dynamic game and how different alignment strategies impact training.
We address these questions by comparing the performance of three strategies: \emph{(i)} using a fixed alignment throughout training, \emph{(ii)} random alignment at each iteration, and \emph{(iii)} adaptive alignment using a multi-armed bandit.
For the latter case, we interpret pulling an arm as selecting one of the alignments, associating round-wise rewards with the validation accuracy at each iteration.

Figure~\ref{fig:bandit-b} depicts the training curves of OCNOpt using different aligning strategies.
For comparison, we include the EKFAC baseline, which corresponds to zeroing out the feedback gains as discussed in Sec.~\ref{sec:ablation}.
The bandit algorithm is instantiated using EXP3++ \cite{seldin2014one}.
While OCNOpt with a fixed alignment already achieves faster convergence compared to EKFAC, dynamic alignment using either random or adaptive strategy leads to further improvement.
Notably, adaptive alignment consistently achieves higher final accuracy values, as demonstrated in Table~\ref{tb:bandit}.
On CIFAR10, the value is improved by 1\% from the baseline and 0.5\% compared to the other two strategies.
These improvements highlight new algorithmic opportunities inspired by \emph{architecture-aware} optimization.

\begin{figure}[!t]
  \centering
  \subfloat{%
    \includegraphics[width=\linewidth]{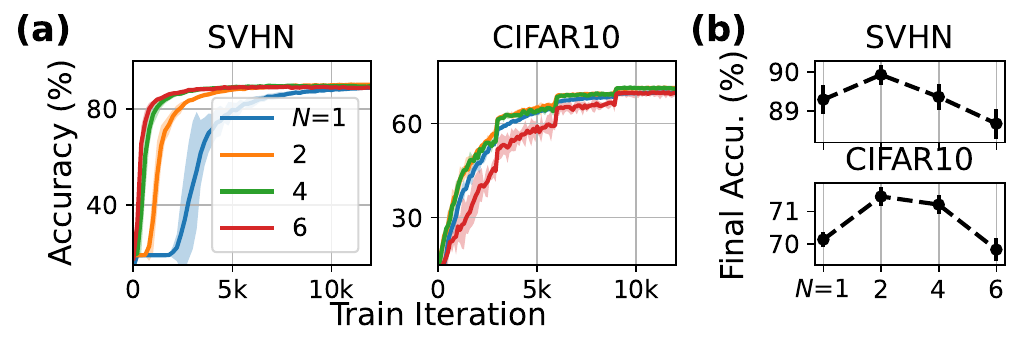}
    \label{fig:multiplayer-a}
  }%
  \subfloat{%
    \textcolor{white}{\rule{1pt}{1pt}}
    \label{fig:multiplayer-b}%
  }%
  \\
  \vskip -0.05in
  \subfloat{%
    \includegraphics[width=\linewidth]{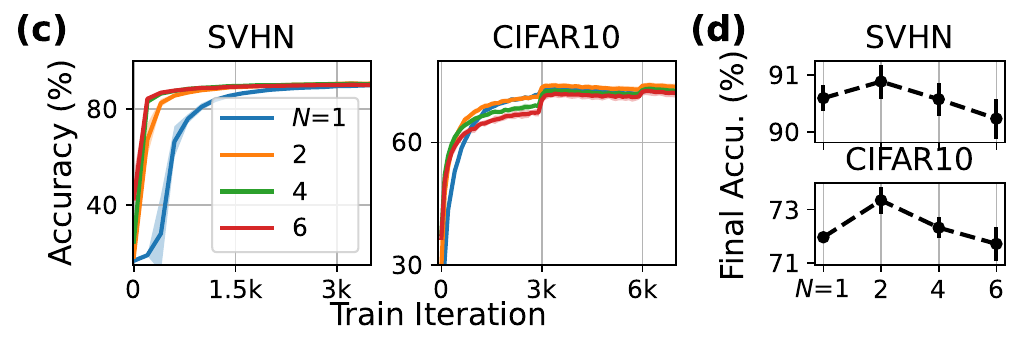}
    \label{fig:multiplayer-c}
  }%
  \subfloat{%
    \textcolor{white}{\rule{1pt}{1pt}}
    \label{fig:multiplayer-d}%
  }%
  \caption{
    Training curves and final accuracies w.r.t. different number of players ($N$) in the fictitious player transformation \eqref{eq:robust-traj} for identity (upper row) and Gauss-Newton (bottom row) curvature approximation, respectively.
  }
  \label{fig:multiplayer}
\end{figure}

\paragraph*{\textit{(ii)} Cooperative training with fictitious players}
The dynamic game interpretation, initially motivated for modern deep architectures, remains valid even when the propagation rule degenerates to feedforward networks \eqref{ff}.
For instance, consider the following transformation:
\begin{align}
    F_k(x_k,\theta_{k,1}, {\cdots},\theta_{k,N}) \coloneqq f_k(x_k,\theta_k), \quad \sum_{n=1}^N \theta_{k,n} = \theta_k.
    \label{eq:robust-traj}
\end{align}
where we intentionally divide the layer's weight into multiple parts.
By treating the propagation rule as multiple players, the game-theoretic framework remains applicable.
Such transformations are prevalent in robust optimal control \cite{pan2015robust,sun2018min},
where the controller (or player in our context)
models external disturbances with \emph{fictitious players},
to enhance robustness or convergence in the optimization process.

Figure.~\ref{fig:multiplayer} presents
the results when OCNOpt assumes different numbers of players $N$ interacting in training a feedforward CNN for image classification.
Here, $N = 1$ corresponds to the original method, and for $N > 1$, we apply the transformation (Eq. 22) and solve for cooperative updates.
It is noteworthy that these fictitious players only appear during the training phase for computing the cooperative updates.
At inference, actions from all players collapse back to $\theta_k$ by summation (\ref{eq:robust-traj}).
Encouraging players to cooperate during training visibly improves both training convergence and test-time accuracy,
as shown in Fig.~\ref{fig:multiplayer-a} and \ref{fig:multiplayer-c}.
However, as indicated in Fig.~\ref{fig:multiplayer-b} and \ref{fig:multiplayer-d}, having more players does not always imply better test-time performance.
In practice, the improvement can slow down or even degrade once $N$ surpasses some critical values.
This suggests that $N$ should be treated as a {hyper-parameter} of {these game-extended transformation}.
Empirically, we find that $N=2$ provides a good trade-off between final performance and convergence speed.

\subsection{Joint optimization of the integration time in Neural ODEs}

Finally, let us discuss an intriguing use case of our OCP framework for optimizing the architecture of Neural ODEs, specifically the integration time $T$.
In practice,
this hyper-parameter can significantly impact on both training time and test-time performance.
Take CIFAR10 for instance (see Fig.~\ref{fig:t1-opt-a}).
The required training time decreases linearly as the integration time decreases from $1$, while the accuracy mostly remains the same unless $T$ becomes too small.
This motivates an interpretation of $T$ as an \emph{architectural parameter} that needs to be jointly optimized during training.
From an OCP standpoint, this can be achieved by including $T$ as an optimizing variable in the Bellman objective $Q(x_t, \theta, T)$ of Neural ODEs in \eqref{Q-cont} and imposing a penalty on having larger $T$.
The derivation follows similarly to that in Sec.~\ref{sec:ddp}; for details, we refer to Appendix~\ref{sec:t1opt} due to space constraint.

\begin{figure}[!t]
  \centering
  \subfloat{%
    \includegraphics[width=\linewidth]{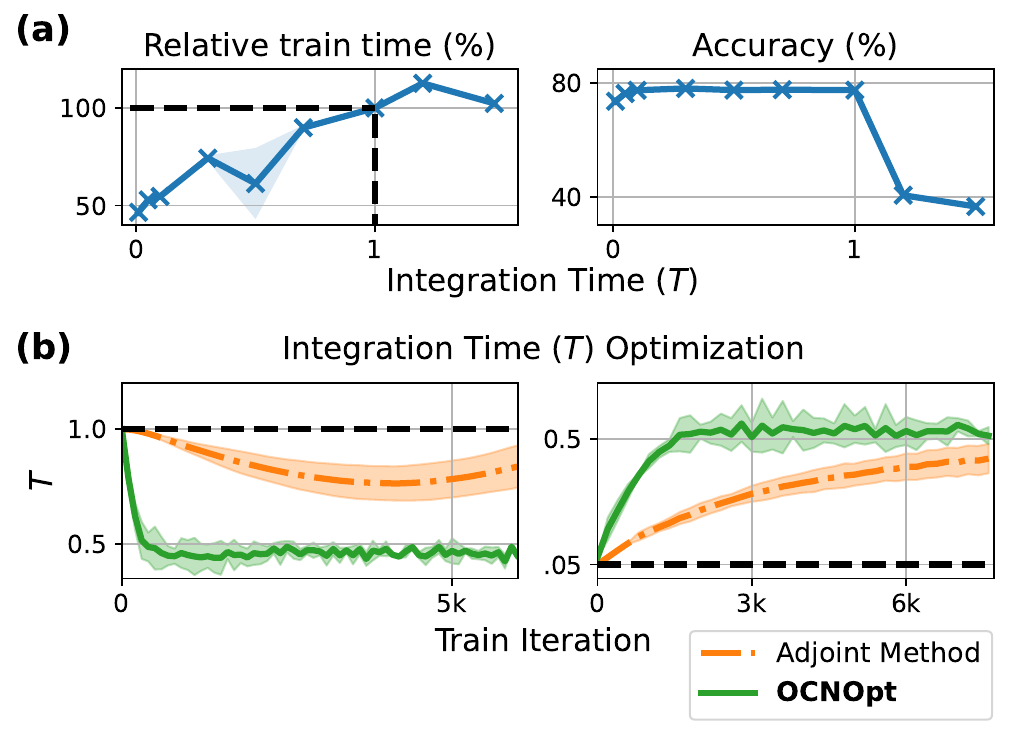}
    \label{fig:t1-opt-a}
  }%
  \subfloat{%
    \textcolor{white}{\rule{1pt}{1pt}}
    \label{fig:t1-opt-b}%
  }%
  \caption{
    (a) Ablation on total training time and test-time accuracies w.r.t. different values of integration time $T$.
    (b) Visualization of how $T$ changes from two initial values in joint optimization. Notice that OCNOpt exhibits stabler convergence compared to baselines such as adjoint method.
  }
  \label{fig:t1-opt}
\end{figure}
\begin{table}[!t]
  \caption{
    Results of integration time joint optimization in Fig.~\ref{fig:t1-opt-b} (left).
  }
  \label{tb:t1-opt}
  \centering
  \begin{tabular}{l!{\vrule width 0.7pt}cccc}
    \toprule
    Method & Relative train time (\%) w.r.t. $T{=}1$ & Accuracy (\%) \\
    \midrule
    Adjoint Method  & 96   & 76.61 \\
    \textbf{OCNOpt} & \textbf{81} & \textbf{77.82} \\
    \bottomrule
    \end{tabular}
\end{table}

Table~\ref{tb:t1-opt} and Fig.~\ref{fig:t1-opt-b}
report the performance of optimizing $T$ and its convergence dynamics, where
we compare OCNOpt to the first-order adjoint baseline proposed in \cite{massaroli2020dissecting}.
It is evident that OCNOpt achieves substantially faster convergence,
and the optimized $T$ hovers stably around $0.5$ without deviation, in contrast to the baseline.
This achieves nearly a 20\% reduction in wall-clock time for a fixed number of training iterations compared to training with fixed $T=1$, without sacrificing test-time accuracy.
We highlight
these improvements as the benefits gained from introducing the OCP principle
to these emerging deep continuous-time models.

\section{Conclusion}

We introduced \textbf{OCNOpt}, a novel class of optimizers
originating from a unique perspective rooted in
dynamical system and optimal control programming (OCP).
By closely examining the algorithmic similarities between existing methods and the DDP algorithm,
we derived algorithmic unification, consolidating OCP as a fundamental discipline for algorithmic design of DNN optimization that combines the strengths of both domains, paving the way for rich algorithmic opportunities and control-theoretic applications.
Through extensive numerical experiments, we demonstrated that
the OCNOpt enhances convergence, robustness, and test-time performance over existing methods.
It is worth noting that OCNOpt, while significantly more efficient than prior OCP-inspired baselines, remains a second-order training method. Bridging its computational gap with first-order methods and extending the OCP principle to other architectures, such as Neural SDEs, PDEs, and Transformers, represent promising directions for future research.

\bibliographystyle{IEEEtran}
\bibliography{main}

\vspace{-40pt}
\begin{IEEEbiography}
[{\includegraphics[width=1in,height=1.25in,clip,keepaspectratio]{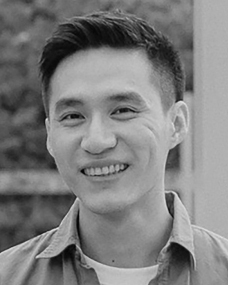}}]{Guan-Horng Liu} received the master degree in Robotics from Carnegie Mellon University, in 2017. He is currently a PhD candidate in Machine Learning at the Georgia Institute of Technology. His research interests lies in scalable computational methods for neural differential equations using stochastic optimal control, higher-order methods, and dynamic optimal transport.\end{IEEEbiography}

\vspace{-50pt}

\begin{IEEEbiography}[{\includegraphics[width=1in,height=1.25in,clip,keepaspectratio]{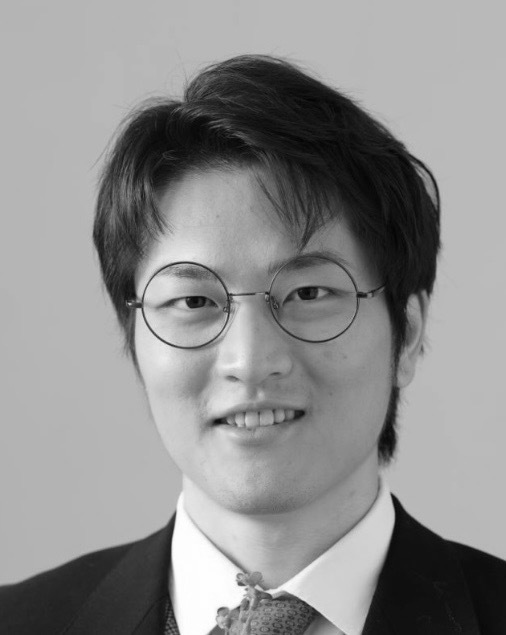}}]{Tianrong Chen} received the master degree in Electrical and Computer Engineering from Georgia Institute of Technology, in 2020. He is currently a PhD candidate in Electrical and Computer Engineering at the Georgia Institute of Technology. His research interests include dynamic optimal transport, scalable computation methods for mean field game, stochastic optimal control and deep learning.
\end{IEEEbiography}

\vspace{-50pt}

\begin{IEEEbiography}[{\includegraphics[width=1in,height=1.25in,clip,keepaspectratio]{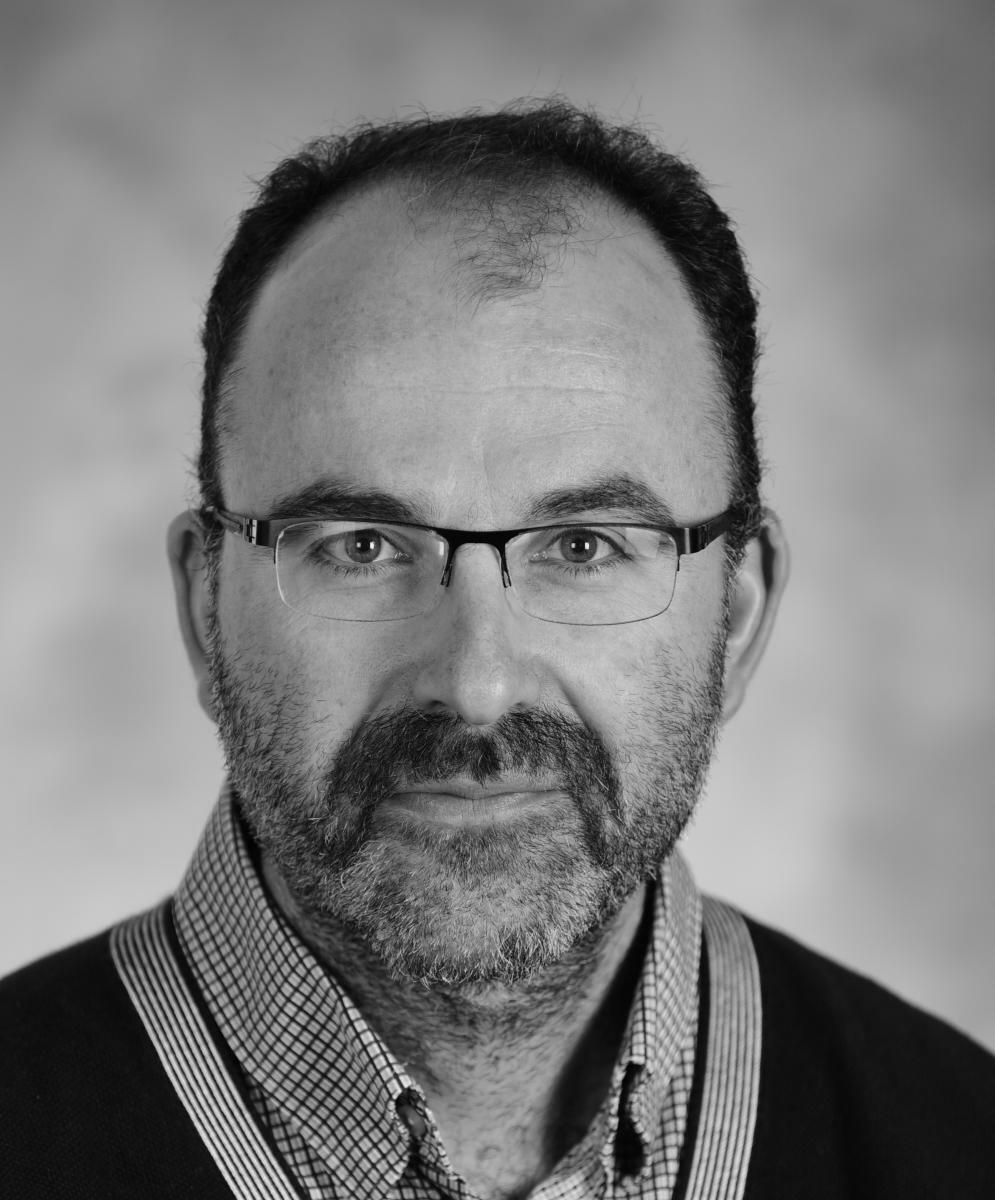}}]{Evangelos A. Theodorou} is an associate professor
in the Daniel Guggenheim School of Aerospace
Engineering at the Georgia Institute of Technology.
Theodorou earned a Bachelors and Masters degree in Electrical Engineering from the Technical University of Crete, Greece, a Masters degree in Computer Science and Engineering from
the University of Minnesota, and a a Ph.D. in
Computer Science from the University of Southern
California. 
His research interests span the areas of stochastic optimal control, robotics, and machine learning.\end{IEEEbiography}

\newpage

\appendices

\section{Proofs} \label{sec:proof}

\subsection{Proof of Theorem~\ref{thm:1}} \label{sec:proof-thm1}

\thmone*

\begin{proof}
The expansion \eqref{bp-Q} can be viewed as the second-order approximation of $Q_k$ from Section~\ref{sec:ddp},
\begin{equation}
  \label{Q-full-expand-app}
  Q_k(\bar x_k, \bar \theta_k) +
  \frac{1}{2}
  \begin{bmatrix}
    1 \\
    \dxt \\
    \dut
  \end{bmatrix}^\T%
  \begin{bmatrix}
    0 & \Qx^\T & \Qu^\T \\
    \Qx & \Qxx & \Qxu \\
    \Qu & \Qux & \Quu
  \end{bmatrix}
  \begin{bmatrix}
    1 \\
    \dxt \\
    \dut
  \end{bmatrix},
\end{equation}
with additional substitution:
\begin{align*}
  \Quu \coloneqq I, \qquad \Qux \coloneqq 0, \qquad \Qxu \coloneqq 0, \qquad \Qxx \coloneqq 0.
\end{align*}
This simplifies the computation of the first-order derivative of the value function $V_k^x$ to
\begin{align*}
  V_k^x 
  &= \Qx - \Qxu \pr{\Quu}^\Inv \Qu \\
  &= \Qx \\
  &= \fx^\T \Vx, 
  \numberthis \label{eq:vl}
\end{align*}
with the terminal condition $V_K^x = \fracpartial{\calL}{x_K}$. That is, the backward dynamics of $V_k^x$ coincide with the Backpropagation,
\begin{equation} \label{eq:bp-app}
  \fracpartial{\calL}{x_k}
  = \fx^\T \fracpartial{\calL}{x_{k+1}},
\end{equation}
and, consequently, the first-order derivative $\Qu$ in \eqref{Qx-Qu} becomes
\begin{align*}
  \Qu 
  &= \lu + \fu^\T \Vx \\
  &= \lu + \fu^\T \fracpartial{\calL}{x_{k+1}} \\
  &= \gamma \bar \theta_k + \fracpartial{\calL}{\theta_k}, \label{eq:ql} \numberthis
\end{align*}
which leads to the desired parameter update:
\begin{align*}
  \delta \theta_k^\star = - \pr{\Quu}^\Inv \Qu \implies
  \theta_k^\star = \bar \theta_k - \pr{\gamma \bar \theta_k + \fracpartial{\calL}{\theta_k}}.
\end{align*}
\end{proof}

\subsection{Proof of Corollary~\ref{coro:2}}

\corotwo*

\begin{proof}

Consider the same second-order expansion in \eqref{Q-full-expand-app} with $\Qux \coloneqq 0$. In this case, the analytic solution simplifies to 
\begin{align*}
    \delta \theta_k^\star(\delta x_k) 
    &= - \pr{\Quu}^\Inv \pr{ \Qu + \Qux \dxt  } \\
    &= - \pr{\Quu}^\Inv \Qu  \label{opt-u-app} \numberthis
\end{align*}
and the derivatives of the value function simplify to
\begin{align*}
  V_k^x 
  &= \Qx - \Qxu \pr{\Quu}^\Inv \Qu \\
  &= \Qx \\
  &= \fx^\T \Vx, 
  \numberthis \label{eq:vl-2}    \\
  V_k^{xx} 
    &= \Qxx - \Qxu \pr{\Quu}^\Inv \Qux \\
    &= \Qxx \\
    &= \fx^\T \Vxx \fx + \fxx \Vx \\
    &= \fracpartial{}{x}\pr{\fx^\T \Vx}
    \numberthis \label{eq:vll-2} 
\end{align*}
with the terminal conditions 
$V_K^{x} = \fracpartial{\calL}{x_K}$
and
$V_K^{xx} = \fracpartial{\calL^2}{x_K \partial x_K}$.
Since \eqref{eq:vl-2} coincides with the Backpropagation dynamics \eqref{eq:bp-app}, we deduct that when $\Qux \coloneqq 0$, it holds that
\begin{equation}
    V_k^x = Q_k^x = \fracpartial{\calL}{x_k},\qquad 
    V_k^{xx} = Q_k^{xx} = \fracpartial{\calL^2}{x_k \partial x_k}.
\end{equation}
This simplifies the first and second-order derivatives, $\Qu$ and $\Quu$, respectively to \eqref{eq:ql} and 
\begin{align*}
  \Quu 
    &= \luu I + \fu^\T \Vxx \fu + \fuu \Vx \\
    &= \luu I + \fu^\T \pr{\fracpartial{\calL^2}{x_{k+1} \partial x_{k+1}}} \fu  + \fuu \fracpartial{\calL}{x_{k+1}} \\
    &= \fracpartial{}{\theta_k} \pr{ \gamma \bar \theta_k + \fu^\T \fracpartial{\calL}{x_{k+1}} } \\
    &= \fracpartial{}{\theta_k} \pr{ \gamma \bar \theta_k + \fracpartial{\calL}{\theta_k} }.
       \numberthis \label{eq:qql}
\end{align*}
Substituting \eqref{eq:ql} and \eqref{eq:qql} back to \eqref{opt-u-app}, we deduct that when $\Qux \coloneqq 0$, the DDP update recovers the Newton's method. Further removing the precondition matrix by $\Quu := I$ reduces the method to gradient descent.

\end{proof}

\noindent
\textbf{Remark.} Comparing the proofs of Theorem~\ref{thm:1} and Corollary~\ref{coro:2}, one may notice that $\Qxx \coloneqq 0$ is an unnecessary condition to reach the same conclusion. This is because, while $\Qxx$ does affect $V^{xx}_k$ and thus the second-order matrices for the preceding layer, the matrices required for computing the update of the preceding layer---namely $Q^{\theta\theta}_{k-1}$ and $Q^{\theta x}_{k-1}$---are replaced respectively by 
\begin{equation*}
    Q^{\theta\theta}_{k-1} \coloneqq I, \qquad Q^{\theta x}_{k-1} \coloneqq 0. 
\end{equation*}
Consequently, the DDP update remains unaffected by the choice of $\Qxx$ and depends solely on $\Qu$, which coincides with the gradient descent, as shown in \eqref{eq:ql}. In Theorem~\ref{thm:1}, we deliberately set $\Qxx \coloneqq 0$ for the compactness.

\newpage 

\subsection{Proof of Proposition~\ref{prop:3}}

\noindent
The following lemma will be useful in proving Proposition~\ref{prop:3}.
\begin{lemma}\label{lm:6}
  Let $V_K^{xx} = y y^\T$, then $V_k^{xx} = r_k r_k^\T$ where $r_k$ follows 
  \begin{align*}
    r_k = \sqrt{ 1 - p_k^\T \pr{\Quu}^{\dagger} p_k  } q_k,~
    q_k = \fx^\T r_{k+1}, ~ p_k = \fu^\T r_{k+1}.
  \end{align*}
\end{lemma}
\begin{proof}
  The condition $V_k^{xx} = r_k r_k^\T$ holds at $k{=}K$ by construction.
  Suppose it also holds at $k+1$, then
  \begin{align*}
    V_k^{xx}
    &= \Qxx - \Qxu \pr{\Quu}^\Inv \Qux \\
    &= \fx^\T \Vxx \fx - \fx^\T \Vxx \fu \pr{\Quu}^\Inv \fu^\T \Vxx \fx \\
    &= \fx^\T r_{k+1} r_{k+1}^\T \fx \\
    &\qquad - \fx^\T r_{k+1} r_{k+1}^\T \fu \pr{\Quu}^\Inv \fu^\T r_{k+1} r_{k+1}^\T \fx \\
    &= q_k q_k^\T - q_k p_k^\T \pr{\Quu}^\Inv p_k q_k^\T \\
    &= q_k \pr{1 - p_k^\T \pr{\Quu}^\Inv p_k} q_k^\T \\
    &= r_k r_k^\T
  \end{align*}
  We conclude the proof by backward induction.
\end{proof}

\propthree*
\begin{proof}
Given Lemma~\ref{lm:6}, it suffices to prove the factorization of $\Qxx$ and $\Qxu$ in \eqref{Qxu-fac}, which can be shown via
\begin{align*}
   \Qxx &= \fx^\T \Vxx \fx = \fx^\T r_{k+1} r_{k+1}^\T \fx = q_k q_k^\T, \\
   \Qxu &= \fx^\T \Vxx \fu = \fx^\T r_{k+1} r_{k+1}^\T \fu = q_k p_k^\T.
\end{align*}
\end{proof}

\subsection{Proof of Corollary~\ref{coro:4}}
\corofour*
\begin{proof}
From Lemma~\ref{lm:6}, $V_k^{xx} = r_k r_k^\T$ where $r_k$ can be expanded as
\begin{align}
  \label{eq:r}
  r_k = \sqrt{ 1 - r_{k+1}^\T \fu \pr{\Quu}^{\dagger} \fu^\T r_{k+1}  } \fx^\T r_{k+1}
\end{align}
with the terminal condition $r_K = \beta \fracpartial{}{x_K}\calL$.
Consider the reparametrization $r_k \coloneqq \alpha_k \fracpartial{\calL}{x_k}$, then \eqref{eq:r} becomes:
\begin{align*}
    &\alpha_k \fracpartial{\calL}{x_k} \\
  =&\sqrt{ 1 {-} \alpha_{k+1}^2 \fracpartial{\calL}{x_{k+1}}^\T \fu \pr{\Quu}^{\dagger} \fu^\T \fracpartial{\calL}{x_{k+1}}} \fx^\T \fracpartial{\calL}{x_{k+1}} \alpha_{k+1}.
\end{align*}
Since $\fracpartial{\calL}{x_k} = \fx^\T \fracpartial{\calL}{x_{k+1}}$ by construction,
it must hold that
\begin{align}
  \label{eq:alpha}
  \alpha_k
= \sqrt{ 1 - \alpha_{k+1}^2 \fracpartial{\calL}{x_{k+1}}^\T \fu \pr{\Quu}^{\dagger} \fu^\T \fracpartial{\calL}{x_{k+1}}} \alpha_{k+1},
\end{align}
with the terminal condition $\alpha_K = \beta$.
Next, recall that
\begin{align*}
  V_k^x
  &= \Qx - \Qxu \pr{\Quu}^\Inv \Qu \\
  &= \fx^\T \Vx - q_k p_k^\T \pr{\Quu}^\Inv \fu^\T \Vx \\
  &= \fx^\T \Vx - \fx^\T r_{k+1} r_{k+1}^\T \fu \pr{\Quu}^\Inv \fu^\T \Vx.
\end{align*}
Substitute $r_k = \alpha_k \fracpartial{\calL}{x_k}$ and introduce $V_k^x \coloneqq \bar\alpha_k \fracpartial{\calL}{x_k}$,
\begin{align*}
    \bar\alpha_k \fracpartial{\calL}{x_k}
  =~&\bar\alpha_{k+1} \fx^\T \fracpartial{\calL}{x_{k+1}} - \alpha_{k+1}^2 \bar\alpha_{k+1} \fx^\T \fracpartial{\calL}{x_{k+1}} \\
  & \cdot  \pr{\fracpartial{\calL}{x_{k+1}}^\T \fu \pr{\Quu}^\Inv \fu^\T \fracpartial{\calL}{x_{k+1}}}.
\end{align*}
Notice that the term inside the bracket is a scalar.
Since $\fracpartial{\calL}{x_k} = \fx^\T \fracpartial{\calL}{x_{k+1}}$,
it must hold that
\begin{align*}
  \bar\alpha_k
&= \bar\alpha_{k+1} - \alpha_{k+1}^2 \bar\alpha_{k+1} \fracpartial{\calL}{x_{k+1}}^\T \fu \pr{\Quu}^\Inv \fu^\T \fracpartial{\calL}{x_{k+1}} \\
&= \bar\alpha_{k+1} \pr{ 1 - \alpha_{k+1}^2 \fracpartial{\calL}{x_{k+1}}^\T \fu \pr{\Quu}^\Inv \fu^\T \fracpartial{\calL}{x_{k+1}} }, \numberthis \label{eq:beta}
\end{align*}
with the terminal condition $\bar\alpha_K = 1$.
Comparing \eqref{eq:beta} to \eqref{eq:alpha}, we can rewrite $\bar\alpha_k$ compactly as
\begin{align*}
  \bar\alpha_k = \prod_{n = k}^{K-1} \pr{\frac{\alpha_{n}}{\alpha_{n+1}}}^2.
\end{align*}
\end{proof}

\subsection{Proof of Proposition~\ref{prop:5}}
\propfive*
\begin{proof}
Recall the continuous-time Bellman objective
\begin{equation*}
 Q_t(x_t, \theta) {\coloneqq} {\int_t^T} \ell(\theta) \dt + \calL\pr{x_t + {\int_t^T} F(\tau, x_\tau,\theta) \rd \tau}.
\end{equation*}
which can be rewritten as
\begin{equation}
  \label{eq:dQt}
  - \fracdiff{}{t} Q_t(x_t, \theta) = \ell(\theta), \quad Q_T(x_T, \theta) = \calL(x_T).
\end{equation}

Given a solution path $(\bar x_t, \bar \theta)$ that solves the ODE, $\fracdiff{}{t} x_t = F(t, x_t, \theta)$, we define the differential variables, $\delta x_t \coloneqq x_t - \bar x_t$ and
$\delta \theta \coloneqq \theta - \bar \theta$, and apply quadratic expansion,
\begin{align*}
  &-\fracdiff{}{t} \left(
Q_t(\bar x_t, \bar \theta) + 
  \frac{1}{2}
  \begin{bmatrix}
    1 \\
    \delta x_t \\
    \delta \theta
  \end{bmatrix}^\T%
  \begin{bmatrix}
    0 & {Q_t^x}^\T & {Q_t^\theta}^\T \\
    Q_t^x & {Q_t^{xx}} & {Q_t^{x\theta}} \\
    {Q_t^\theta} & {Q_t^{\theta x}} & {Q_t^{\theta\theta}}
  \end{bmatrix}
  \begin{bmatrix}
    1 \\
    \delta x_t \\
    \delta \theta
  \end{bmatrix}\right) \\
  &\qquad= \ell(\bar \theta) + {\ell^\theta}^\T \delta \theta + \frac{1}{2}\delta \theta^\T \ell^{\theta\theta} \delta \theta. \numberthis \label{eq:dQ2}
\end{align*}
The LHS of \eqref{eq:dQ2} involves the time derivative of the expansion.
These computations follow standard ordinary calculus, \eg
\begin{align*}
  &\fracdiff{}{t} \pr{ \delta x_t^\T Q_t^{x\theta} \delta \theta } \\
    =&
    {\fracdiff{\delta x_t}{t}}^\T Q_t^{x\theta}  \delta \theta
    + \delta x_t^\T \pr{\fracdiff{}{t} Q_t^{x\theta}} \delta \theta  + \delta x_t^\T Q_t^{x\theta} \fracdiff{\delta \theta}{t} \\
    =& \pr{F_t^{x\T}\delta x_t + F_t^{\theta\T}\delta \theta} Q_t^{x\theta}  \delta \theta + \delta x_t^\T \pr{\fracdiff{}{t} Q_t^{x\theta}} \delta \theta, 
\end{align*}
where in the second equality we expand 
the dynamics of the differential variables $\delta x_t$ and $\delta \theta$ up to the first order,
\begin{align}
  \label{eq:ddx}
  \fracdiff{\delta x_t}{t} = F_t^{x\T}\delta x_t + F_t^{\theta\T}\delta \theta, \quad
  \fracdiff{\delta \theta}{t} = 0.
\end{align}

Repeating similar computations for the remaining terms on the LHS of \eqref{eq:dQ2} and substituting $\ell \coloneqq \frac{\gamma}{2}\norm{\theta}^2$ into the RHS of \eqref{eq:dQ2}, we obtain the following expression for the time evolution of the second-order matrices:
\begin{subequations}
  \label{eq:dQ}
  \begin{align}
  - \fracdiff{Q_t^{xx}}{t}           &= F_t^{x\T}Q_t^{xx} + Q_t^{xx}F_t^{x}, \quad Q_T^{xx} = \fracpartial{\calL^2}{x_T \partial x_T} \label{eq:dQxx} \\
  - \fracdiff{Q_t^{x\theta}}{t}      &= F_t^{x\T}Q_t^{x\theta} + Q_t^{xx}F_t^{\theta}, \quad Q_T^{x\theta} = 0 \label{eq:dQxu} \\
  - \fracdiff{Q_t^{\theta\theta}}{t} &= F_t^{\theta\T}Q_t^{x\theta} + Q_t^{\theta x}F_t^{\theta} + \gamma I, \quad Q_T^{\theta\theta} = 0 \label{eq:dQuu}
  \end{align}
\end{subequations}

Finally, one can verify via backward induction that
when $Q_T^{xx} = \fracpartial{\calL}{x_T}\fracpartial{\calL}{x_T}^\T$,
the solutions to these ODEs \eqref{eq:dQ} can be factorized into
\begin{align*}
  Q_t^{xx} = q_t q_t^\T, \quad
  Q_t^{x\theta} = q_t p_t^\T, \quad
  Q_t^{\theta\theta} = p_t p_t^\T + \gamma\cdot(T-t)I,
\end{align*}
where $q_t$ and $p_t$ follow \eqref{qt-cont}.
This is due to the fact that
\begingroup
\allowdisplaybreaks
\begin{align*}
  - \fracdiff{Q_t^{xx}}{t}
  &= F_t^{x\T}Q_t^{xx} + Q_t^{xx}F_t^{x}
  = F_t^{x\T}q_t q_t^\T + q_t q_t^\T F_t^{x} \\
  &= - \fracdiff{q_t}{t} q_t^\T - q_t \fracdiff{q_t}{t}^\T
  = - \fracdiff{}{t} \pr{q_t q_t^\T} \\
  - \fracdiff{Q_t^{x\theta}}{t}
  &= F_t^{x\T}Q_t^{x\theta} + Q_t^{xx}F_t^{\theta}
  = F_t^{x\T}q_t p_t^\T + q_t q_t^\T F_t^{\theta} \\
  &= - \fracdiff{q_t}{t} p_t^\T - q_t \fracdiff{p_t}{t}^\T
  = - \fracdiff{}{t} \pr{q_t p_t^\T} \\
  - \fracdiff{Q_t^{\theta\theta}}{t}
  &= F_t^{\theta\T}Q_t^{x\theta} + Q_t^{\theta x}F_t^{\theta} + \gamma I \\
  &= F_t^{\theta\T}q_t p_t^\T + p_t q_t^\T F_t^{\theta}  + \gamma I \\
  &= - \fracdiff{p_t}{t} p_t^\T - p_t \fracdiff{p_t}{t}^\T  + \gamma I
  = - \fracdiff{}{t} \pr{p_t p_t^\T}  + \gamma I.
\end{align*}%
\endgroup

\end{proof}

\section{Additional Discussions}

\subsection{Extension of Proposition~\ref{prop:3} to Higher Rank} \label{sec:prop3-ext}

It is possible to extend Proposition~\ref{prop:3} to higher-rank factorization with rank $N$, where
\begin{align*}
  V_K^{xx} = \sum_i^N y^i y^{i\T}.
\end{align*}
Adopting similar derivation, suppose $V_{k+1}^{xx} = \sum_i^N r_{k+1}^i r_{k+1}^{i\T}$ holds at $k+1$, then
\begin{align*}
  \Qxx &= \fx^\T \Vxx \fx = \fx^\T \pr{\sum_i^N r_{k+1}^i r_{k+1}^{i\T}} \fx = \sum_i^N q_k^i q_k^{i\T}, \\
  \Qxu &= \fx^\T \Vxx \fu = \fx^\T \pr{\sum_i^N r_{k+1}^i r_{k+1}^{i\T}} \fu = \sum_i^N q_k^i p_k^{i\T},
\end{align*}
where
\begin{align*}
  q_k^i \coloneqq \fx^\T r_{k+1}^i, \quad p_k^i \coloneqq \fu^\T r_{k+1}^i
\end{align*}
are defined similarly to \eqref{qpr}. Therefore, we obtain
\begin{align*}
  V_k^{xx}
  &= \Qxx - \Qxu \pr{\Quu}^\Inv \Qux \\
  &= \sum_i^N q_k^i q_k^{i\T} - \pr{\sum_i^N q_k^i p_k^{i\T}} \pr{\Quu}^\Inv \pr{\sum_i^N p_k^i q_k^{i\T}} \\
  &= \sum_i^N q_k^i q_k^{i\T} - \sum_i^N \sum_j^N q_k^i p_k^{i\T} \pr{\Quu}^\Inv p_k^j q_k^{j\T}.
  \numberthis \label{eq:V2}
\end{align*}

\Eqref{eq:V2} can be presented in a compact matrix form
\begin{align*}
  V_k^{xx} = \mQ_k \mS_k \mQ_k^\T
\end{align*}
where $\mQ_k = [q_k^1, \cdots, q_k^N] \in \R^{d \times N}$ and the element of $\mS_k$ given by
$[\mS_k]_{ij} = p_k^{i\T} \pr{\Quu}^\Inv p_k^j + [I]_{ij}$, where $[I]_{ij}$ is the $(i,j)$ element of the identity matrix.

Since $\mS_k$ is a real symmetric matrix, we know $\mS_k = \mU_k \Sigma_k \mU_k^\T$ for some orthogonal matrix $\mU_k$ and diagonal matrix $\Sigma_k$.
\begin{align*}
  V_k^{xx} = \mQ_k \mU_k \Sigma_k \pr{ \mQ_k \mU_k}^\T = \sum_i^N \sigma_{ii} v_{k}^i v_{k}^{i\T} = \sum_i^N r_{k}^i r_{k}^{i\T},
\end{align*}
where $v_{k}^i$ is the $i$-th column of $\mQ_k \mU_k$, $\sigma_{ii}$ is the $i$-th diagonal entry of $\Sigma_k$, and finally we define $r_{k}^i \coloneqq \sqrt{\sigma_{ii}} v_{k}^i$.

\subsection{Kronecker-based Approximation of Preconditioned Matrix} \label{sec:approx-explain}

Here, we detail the computation in Section~\ref{sec:4.C}. We first restate \eqref{kfac} and \eqref{kfac-u} below:
\begin{align}
  \E\br{ \Qu \Qu^\T}
  &= \E\br{ \pr{x_k \otimes V^h_k} \pr{x_k \otimes V^h_k}^\T } \nonumber \\
  &= \E\br{ \pr{x_k x_k^\T} \otimes \pr{V^h_k {V^h_k}^\T} } \nonumber \\
  &\approx \E\br{x_k x_k^\T} \otimes \E\br{V^h_k {V^h_k}^\T}, \tag{\ref{kfac}} \\
    \pr{\Quu}^\dagger
    &\approx \textstyle \frac{1}{\gamma} I + \pr{\E\br{x_k x_k^\T} \otimes \E\br{V^h_k {V^h_k}^\T}}^\dagger \nonumber \\
     &= (U_1 \otimes U_2) (\textstyle \frac{1}{\gamma} I + \Sigma_1^\dagger \otimes \Sigma_2^\dagger) (U_1^\T \otimes U_2^\T).  \tag{\ref{kfac-u}}
\end{align}

The approximation in \eqref{kfac} replaces the averaging of the Kronecker product of two matrices, $x_kx_k^\T$ and $V_k^xV_k^{x\T}$, with the Kronecker product of their averaged matrices, $\E[x_kx_k^\T]$ and $\E[V_k^xV_k^{x\T}]$. Such an approximation has been extensively explored in second-order DNN training \cite{martens2014new,grosse2016kronecker,martens2015optimizing}, and has been empirically shown to yield reasonable approximation. 
Fundamentally, it assumes statistical independence between the two matrices.

To compute the pseudo inversion of \eqref{kfac} and to account for possible regularization $\gamma \ge 0$, we first apply eigen-decomposition to the final two matrices in \eqref{kfac}:
\begin{align*}
    \E\br{x_k   x_k^\T} = U_1 \Sigma_1 U_1^\top, \quad 
    \E\br{V_k^x V_k^{x\T}} = U_2 \Sigma_2 U_2^\top.
\end{align*}
It then follows from standard Kronecker-based computation that
\begin{align*}
    &~~~\pr{\E\br{x_k x_k^\T} \otimes \E\br{V^h_k {V^h_k}^\T}}^\dagger \\
    &= \pr{ \pr{U_1 \Sigma_1 U_1^\top} \otimes \pr{U_2 \Sigma_2 U_2^\top} }^\dagger \\
    &= \pr{ \pr{U_1 \Sigma_1 U_1^\top}^\dagger \otimes \pr{U_2 \Sigma_2 U_2^\top}^\dagger } \\
    &= \pr{U_1 \Sigma_1^\dagger U_1^\top} \otimes \pr{U_2 \Sigma_2^\dagger U_2^\top} \\
    &= \pr{U_1 \otimes U_2} \pr{\Sigma_1^\dagger \otimes \Sigma_2^\dagger} \pr{U_1^\top \otimes U_2^\top}
\end{align*}
Adding back the regularization $\frac{1}{\gamma}I$ by noticing the fact that 
\begin{align*}
    \pr{U_1 \otimes U_2} \pr{U_1^\top \otimes U_2^\top}
    = \pr{U_1 U_1^\top} \otimes \pr{U_2 U_2^\top} = I
\end{align*}
leads to the result in \eqref{kfac-u}.

\subsection{Markovian Construction for Residual Layer} \label{sec:res-app}

For completeness, the propagation rule for the $k$-th and $(k+2)$-th layers in Fig.~\ref{fig:resnet} can be made Markovian via
\begin{align*}
    x_k &= F_{k-1}(x_{k-1}, \theta_{k-1}) = 
    \begin{bmatrix}
        f_{k-1}(z_{k-1}, \theta_{k-1}) \\ 
        z_{k-1}
    \end{bmatrix}, \\
    x_{k+2} &= F_{k+1}(x_{k+1}, \theta_{k+1}) = 
    \mathrm{ReLU}(f_{k+1}(z_{k+1}, \theta_{k+1}) + z_{k,\text{res}}).
\end{align*}

\subsection{Integration Time Joint Optimization} \label{sec:t1opt}
\noindent
Recall the continuous-time Bellman objective in \eqref{Q-cont},
\begin{align}
  \tag{\ref{Q-cont}}
 Q_t(x_t, \theta) {\coloneqq} {\int_t^T} \ell(\theta) \dt + \calL\pr{x_t + {\int_t^T} F(\tau, x_\tau,\theta) \rd \tau}.
\end{align}
Consider an extension \eqref{Q-cont} in which the integration time $T$ is introduced as a variable:
\begin{align}
  \label{eq:Qt1}
  Q_t(x_t, \theta, T) = \int_t^T \ell(\theta) \dt + \widetilde\calL(x_T, T),
\end{align}
where we set $\widetilde\calL(x_T, T) \coloneqq \calL(x_T) + \frac{c}{2}T^2$ for some hyper-parameter $c$. Similar to \eqref{eq:dQt}, we can rewrite \eqref{eq:Qt1} as
\begin{align}
  \label{eq:dQt1}
  - \fracdiff{}{t} Q_t(x_t, \theta, T) = \ell(\theta), \quad Q_T(x_T, \theta, T) = \widetilde\calL(x_T, T)
\end{align}
and follow a second-order expansion similar to \eqref{eq:dQ2}, except w.r.t. additionally to the differential integration time $\delta T$.
After substituting \eqref{eq:ddx} and additionally $\fracdiff{\delta T}{t} = 0$, we obtain the following ODEs:
\begin{equation}
  \label{eq:dQt11}
  \begin{split}
  - \fracdiff{Q_t^T}{t} &= 0, \quad
  - \fracdiff{Q_t^{TT}}{t} = 0, \\
  - \fracdiff{Q_t^{Tx}}{t} &= Q_t^{Tx} \Fx, \quad
  - \fracdiff{Q_t^{T\theta}}{t} = Q_t^{Tx} \Fu.
\end{split}
\end{equation}

The derivation of the terminal conditions of \eqref{eq:dQt11} is more involved.
First, perturbing the integration time $T$ by an infinitesimal amount $\delta T$ yields
\begin{align*}
  Q_T(x_{T + \delta T}, \theta, T + \delta T)
    &= \ell(\theta) \delta T + \widetilde\calL(x_{T + \delta T}, T + \delta T).
\end{align*}
The derivatives of $\widetilde\calL$ w.r.t. $T$ take the forms:
\begin{align*}
  \fracdiff{}{T}\widetilde\calL &= \fracpartial{}{T}\widetilde\calL + F^\T \fracpartial{}{x_T}\widetilde\calL \equiv \widetilde\calL^T + F^\T \widetilde\calL^x, \\
  \fracdiff{}{T} \pr{\fracdiff{}{T}\widetilde\calL} &= \widetilde\calL^{TT} + F^\T\widetilde\calL^{xT} + \widetilde\calL^{xT\T} F + F^\T  \widetilde\calL^{xx} F \\
  \fracdiff{}{x_T} \pr{\fracdiff{}{T}\widetilde\calL} &= \widetilde\calL^{xT} + \widetilde\calL^{xx} F, \quad \fracdiff{}{\theta} \pr{\fracdiff{}{T}\widetilde\calL} = 0.
\end{align*}
Consequently, the terminal conditions of \eqref{eq:dQt11} are given by
\begin{equation}\label{eq:dQt111}\begin{split}
  Q_T^T    &= \ell + \widetilde\calL^T + F^\T \widetilde\calL^x, \\
  Q_T^{TT} &= \widetilde\calL^{TT} + F^\T\widetilde\calL^{xT} + \widetilde\calL^{xT\T} F + F^\T  \widetilde\calL^{xx} F, \\
  Q_T^{Tx} &= \widetilde\calL^{Tx} + F^\T \widetilde\calL^{xx}, \\
  Q_T^{T\theta} &= 0
\end{split}\end{equation}
With \eqref{eq:dQt111}, we can solve the ODEs \eqref{eq:dQt11} then construct the update rule for the integration time as
\begin{align}
  \delta T^\star(\delta x, \delta \theta) =
  - \pr{Q_0^{TT}}^{\Inv} \pr{ Q_0^{T} + Q_0^{Tx}\delta x + Q_0^{T\theta}\delta \theta }.
\end{align}
In practice, we found that dropping $\delta x$ tends to stabilize training. We set $\delta \theta$ as the parameter update.

\begin{figure}[!t]
  \centering
    \includegraphics[width=.65\linewidth]{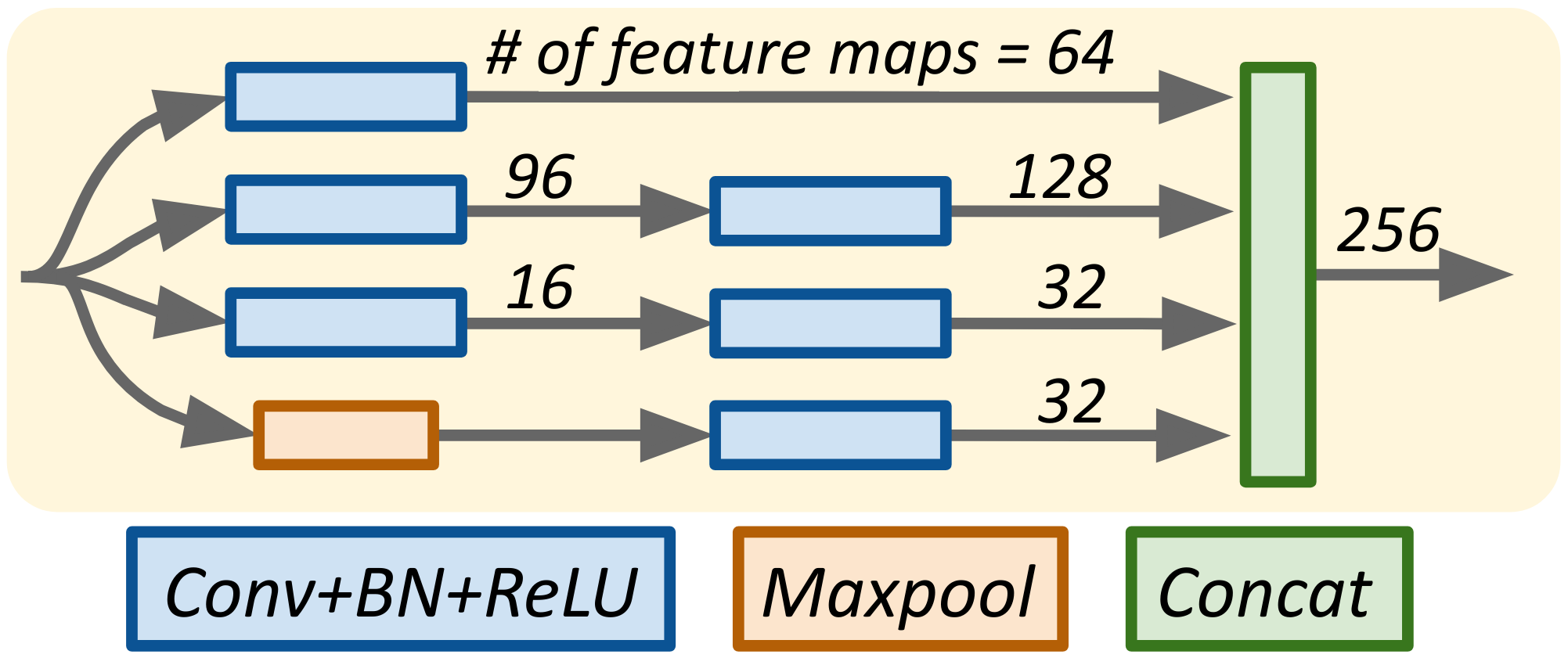}
  \caption{
    Illustration of an inception block.
  }
  \label{fig:inception}
\end{figure}
\begin{table}[!t]
  \caption{
    Hyper-parameter grid search for each baseline.
  }
  \label{tb:hyper}
  \centering
  \begin{tabular}{l!{\vrule width 0.7pt}cccc}
    \toprule
     & Learning rate $\eta$ & Weight regularization $\gamma$ \\
    \midrule
    SGD   & \specialcell{\{ 3e-3, 5e-3, 7e-3, 1e-2, 3e-2, \\ 5e-2, 7e-2, 1e-1, 3e-1, 5e-1 \}}  & \{0.0, 1e-4, 1e-3 \} \\[10pt]
    \specialcell{Adam \& \\ RMSProp}  & \specialcell{\{ 5e-4, 7e-4, 1e-3, 3e-3, 5e-3, \\ 7e-3, 1e-2, 3e-2, 5e-2 \}} & \{0.0, 1e-4, 1e-3 \} \\[10pt]
    EKFAC & \specialcell{\{ 3e-3, 5e-3, 7e-3, 1e-2, 3e-2, \\ 5e-2, 7e-2, 1e-1, 3e-1 \}} & \{0.0, 1e-4, 1e-3 \} \\
    \bottomrule
    \end{tabular}
\end{table}
\begin{table}[!t]
  \caption{
    Curvature approximation adopted for OCNOpt in Table~\ref{tb:ddp-clf}. We consider identity, adaptive \eqref{adap}, or Gauss-Newton \eqref{kfac} approximations; for further details, refer to Table~\ref{table:update-rule}.
  }
  \label{tb:curvature}
  \centering
  \begin{tabular}{r!{\vrule width 0.5pt}r!{\vrule width 1pt}c}
    \toprule
    Network Type & Dataset & \specialcell{Curvature approximation \\ used in OCNOpt} \\
    \midrule
    \multirow{4}{*}{Fully-Connected}
    & WINE & adaptive \\[1pt]
    & DIGITS & adaptive \\[1pt]
    & MNIST & adaptive \\[1pt]
    & FMNIST & adaptive \\[1pt]
    \midrule
    \multirow{3}{*}{Convolutional}
    & MNIST & adaptive \\[1pt]
    & SVHN & Gauss-Newton \\[1pt]
    & CIFAR10 & Gauss-Newton \\[1pt]
    \midrule
    \multirow{4}{*}{Residual}
    & MNIST & Gauss-Newton \\[1pt]
    & SVHN & Gauss-Newton \\[1pt]
    & CIFAR10 & adaptive \\[1pt]
    & CIFAR100 & Gauss-Newton \\[1pt]
    \midrule
    \multirow{3}{*}{Inception}
    & MNIST & identity \\[1pt]
    & SVHN & Gauss-Newton \\[1pt]
    & CIFAR10 & Gauss-Newton \\[1pt]
    \bottomrule
    \end{tabular}
\end{table}

\section{Experiment Detail} \label{app:d}

\subsection{Network Architectures}

\textbf{Table~\ref{tb:ddp-clf}.}
We employ ReLU activation for all datasets, except Tanh for WINE and DIGITS to better distinguish between optimizers.
The implementation of ResNet18 is adopted from \url{https://pytorch.org/hub/pytorch_vision_resnet/}.
The inception networks consist of a convolutional layer followed by an inception block, as illustrated in Fig.~\ref{fig:inception}, another convolutional layer, and finally two fully-connected layers.

\textbf{Table~\ref{tb:snopt}.}
We use the same networks as Chen et al. \cite{chen2018neural} and Kidger et al. \cite{kidger2020neural}, parametrizing $F(t, x_t, \theta)$ as CNNs and FCNs, respectively, for image classification and time-series prediction.
For CNF, we consider the same FCNs as Grathwohl et al. \cite{grathwohl2019ffjord}.

\textbf{Figure~\ref{fig:robust}.}
We employ the same inception network as in Table~\ref{tb:ddp-clf}.

\textbf{Figure~\ref{fig:bandit}} and \textbf{Table~\ref{tb:bandit}.}
For CIFAR10, we use the same ResNet18 architecture, while for SVHN, we employ a smaller residual network with only one residual block.

\textbf{Figure~\ref{fig:multiplayer}.}
The CNN comprises four convolutional layers with 3$\times$3 kernels and 32 feature maps, followed by two fully-connected layers, all activated by ReLU.

\textbf{Figure~\ref{fig:t1-opt}} and \textbf{Table~\ref{tb:t1-opt}.}
We utilize the same network as presented in Table~\ref{tb:snopt}.

\subsection{Baselines}
\noindent
For each baseline, we select its own hyper-parameters from an appropriate search space, detailed in Table~\ref{tb:hyper}. We utilize the official implementation from \url{https://github.com/Thrandis/EKFAC-pytorch} for EKFAC and implement our own E-MSA in PyTorch due to the lack of GPU support in the official code released by Li et al. \cite{li2017maximum}.

\subsection{Other Implementation Details}
\noindent
All experiments are conducted using PyTorch on GPU machines, including one GTX 1080 TI, one GTX 2070, one RTX TITAN, and four Tesla V100 SXM2 16GB. Numerical values reported in Table~\ref{tb:ddp-clf} are averaged over 4-10 random trials, whereas those in Table~\ref{tb:snopt} are averaged over 3 random trials. Finally, Table~\ref{tb:curvature} details the curvature approximation employed in our OCNOpt for each combination of dataset and network specified in Table~\ref{tb:ddp-clf}.

\section{Additional Experiments} \label{app:e}

\begin{table*}[!t]
    \caption{
        Performance comparison similar to Table~\ref{tb:snopt} using the Adams-Moulton method, \\ as opposed to the Runge-Kutta 4(5) method, as the numerical ODE integrator.
        \label{tb:snopt2}
    }
    \centering
    \begin{tabular}
        {lccccccccc}
    \toprule
    & \multicolumn{3}{c}{Image Classification} & \multicolumn{3}{c}{Time-Series Prediction} & \multicolumn{3}{c}{Continuous Normalizing Flow} \\[2pt]
    & {MNIST}  & {SVHN} & {CIFAR10}
    & {SpoAD}  & {ArtWR} & {CharT}
    & Circle & {Gas}  & {Miniboone} \\
    \midrule
    Adam & 98.86  &  91.76  &  77.22  &  95.33  &  86.28  &  88.83  &  0.90  &  -6.51  &  13.29 \\[2pt]
    SGD  & 98.71  &  94.19  &  76.48  &  \textbf{97.80}  &  87.05  &  95.38  &  0.93  &  -4.69  &  13.77 \\[1pt]
    \midrule
    \specialcelll{ \textbf{OCNOpt} \\ (\textbf{ours})}
                  & \textbf{98.95} &  \textbf{95.76} &  \textbf{79.00}
                  & 97.45          &  \textbf{89.50} &  \textbf{97.17}
                  & \textbf{0.86}  &  \textbf{-7.41} &  \textbf{12.37} \\
    \bottomrule
    \end{tabular}
\end{table*}

\begin{figure*}[!t]
    \centering
    \includegraphics[width=.99\linewidth]{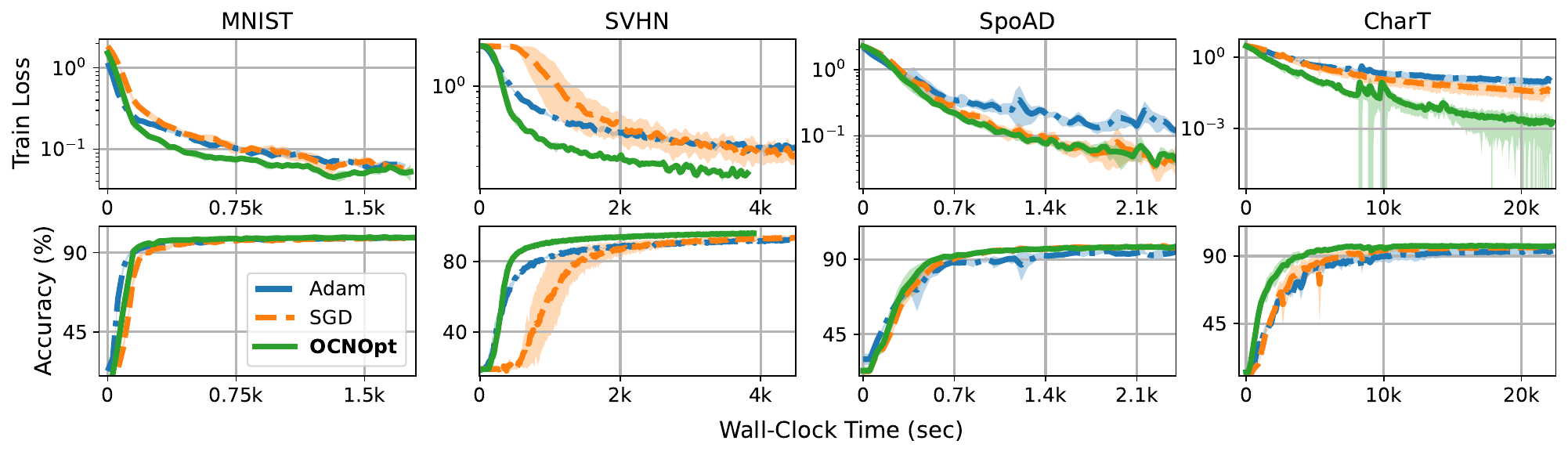}
    \caption{
        Convergence as a function of \emph{wall-clock} runtime for training Neural ODEs on other image classification and time-series prediction datasets.
        Note that the ``k'' in the x-axis abbreviates ``$1000$''.
        \label{fig:convergence2}
    }
\end{figure*}

\begin{figure}[!t]
  \centering
    \includegraphics[width=.5\linewidth]{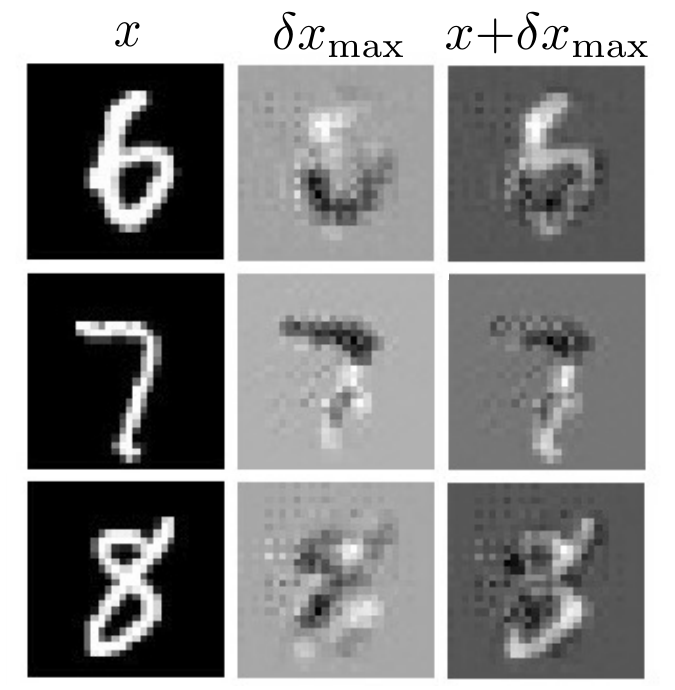}
  \caption{
    Visualization of the feedback policies on MNIST.
  }
  \label{fig:feedback-viz}
\end{figure}

\noindent
Table~\ref{tb:snopt2} presents a similar performance comparison using the Adams-Moulton method, as opposed to the Runge-Kutta 4(5) method in Table~\ref{tb:snopt}, as the numerical ODE integrator. The results suggest that the superior performance of OCNOpt remains consistent when using other ODE integrators.

Figure~\ref{fig:convergence2} illustrates convergence results on other datasets related to Neural ODEs.

Figure~\ref{fig:feedback-viz} attempts to visualize the feedback policies when training a CNN on MNIST. Specifically, we performed singular-value decomposition on the feedback gain $(\Quu)^{\Inv}\Qux$ and then projected the leading right-singular vector back to image space. Consequently, these feature maps, denoted by $\delta x_{\max}$ in Fig.~\ref{fig:feedback-viz}, correspond to the most responsive differences to which the policies will pay attention during training. Intriguingly, the policies seem to capture meaningful differences between visually similar classes, such as 7 and 1 (second row), or 8 and 3 (third row).

\end{document}